%% file: main.tex
\documentclass{article}
\input{math_commands}

\usepackage{amsmath,amssymb,amsthm,amsfonts,bm,amsthm,color}
\usepackage{bbm}
\usepackage{enumitem}
\usepackage[normalem]{ulem}
\usepackage{pifont}
\usepackage{wrapfig}
\usepackage{appendix}
\usepackage{pgfplots}
\pgfplotsset{compat=1.18}

\usepackage{caption}
\usepackage{subcaption}

\usepackage{natbib}
\setcitestyle{authoryear,round}

\usepackage{tcolorbox}
\usepackage{fullpage}
\usepackage[utf8]{inputenc} 
\usepackage[T1]{fontenc}    
\usepackage{hyperref}       
\usepackage{url}            
\usepackage{booktabs}       
\usepackage{amsfonts}       
\usepackage{nicefrac}       
\usepackage{microtype}      
\usepackage{xcolor}         
\usepackage[nameinlink, capitalize]{cleveref}
\AtBeginEnvironment{appendices}{\crefalias{section}{appendix}}
\AtBeginEnvironment{appendices}{\crefalias{subsection}{appendix}}

\newtheorem{theorem}{Theorem}[section]
\newtheorem{corollary}{Corollary}
\newtheorem{proposition}{Proposition}
\newtheorem{lemma}[theorem]{Lemma}

\newtheorem{remark}{Remark}

\newcommand{\vdelta}{\boldsymbol{\delta}}

\DeclareMathOperator{\BV}{BV}

\DeclareMathOperator{\sgn}{sgn}

 \def\vc{\boldsymbol{c}} 
\def\T{{\top}}
\def\btheta{{\boldsymbol{\theta}}}

\newcommand{\jupyter}[1]{\href[pdfnewwindow=true]{#1}{\smash{\begingroup
\setbox0=\hbox{\includegraphics[height=1.5em]{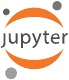}}%
\parbox{\wd0}{\box0}\endgroup}}}

\title{\textbf{
The Effects of Multi-Task Learning on ReLU Neural Network Functions}}
\author{
Julia Nakhleh \\
University of Wisconsin-Madison\\
\texttt{jnakhleh@wisc.edu}\\
\and
Joseph Shenouda\\
University of Wisconsin-Madison\\
\texttt{jshenouda@wisc.edu} \\
\and
Robert D. Nowak\\
University of Wisconsin-Madison\\
\texttt{rdnowak@wisc.edu}\\
}
\date{}

\begin{document}

\maketitle

\begin{abstract}
This paper studies the properties of solutions to multi-task shallow ReLU neural network learning problems, wherein the network is trained to fit a dataset with minimal sum of squared weights. Remarkably, the solutions learned for each individual task resemble those obtained by solving a kernel regression problem, revealing a novel connection between neural networks and kernel methods. It is known that single-task neural network learning problems are equivalent to a minimum norm interpolation problem in a non-Hilbertian Banach space, and that the solutions of such problems are generally non-unique. In contrast, we prove that the solutions to univariate-input, multi-task neural network interpolation problems are almost always unique, and coincide with the solution to a minimum-norm interpolation problem in a Sobolev (Reproducing Kernel) Hilbert Space. We also demonstrate a similar phenomenon in the multivariate-input case; specifically, we show that neural network learning problems with large numbers of tasks are approximately equivalent to an $\ell^2$ (Hilbert space) minimization problem over a fixed kernel determined by the optimal neurons.
\end{abstract}

\section{Introduction}
This paper characterizes the functions learned by multi-output shallow ReLU neural networks trained with weight decay regularization, wherein each network output fits a different ``task'' (i.e., a different set of labels on the same data points). We show that the solutions to such multi-task training problems can differ dramatically from those obtained by fitting separate neural networks to each task individually. Unlike standard intuitions \cite{caruana1997multitask} and existing theory \cite{ben2003exploiting,maurer2016benefit} regarding the effects and benefits of multi-task learning, our results do not rely on similarity between tasks. 

We focus on shallow, vector-valued (multi-output) neural networks with Rectified Linear Unit (ReLU) activation functions, which are functions $f_{\btheta}:\R^d\rightarrow \R^T$  of the form
\begin{eqnarray} \label{eq:nn}
    f_{\btheta}(\vx) \ = \ \sum_{k=1}^K \vv_k \big(\vw_k^\T {\vx}+b_k\big)_+ \ + \ \mA \vx + \vc
    \label{relu_net}
\end{eqnarray}
where $(\cdot)_+ = \max\{0,\cdot\}$ is the ReLU activation function, 
$\vw_k \in \R^d$, $\vv_k\in\R^T$, and $b_k\in\R$ are the input and output weights and bias of the $k^{\text{th}}$ neuron.  $K$ is the number of neurons and $T$ denotes the number of tasks (outputs) of the neural network. The affine term $\mA \vx + \vc$ is the residual connection (or skip connection), where $\mA \in \R^{T\times d}$ and $\vc \in \R^T$. The set of all parameters is denoted by $\btheta:=\big(\{\vv_k,\vw_k,b_k\}_{k=1}^K,\mA,\vc\big)$. 

Neural networks are trained to fit data using gradient descent methods and often include a form of regularization called \emph{weight decay}, which penalizes the $\ell^2$ norm of the network weights.  We consider weight decay applied only to the input and output weights of the neurons---no regularization is applied to the biases or residual connection. This is a common setting studied frequently in past work \cite{savarese2019infinite,ongie2019function,parhi2021banach}. Intuitively, only the input and output weights---not the biases or residual connection---affect the ``regularity'' of the neural network function as measured by its second (distributional) derivative, which is why it makes sense to regularize only these parameters. Given a set of training data points $(\vx_1, \vy_1), \dots, (\vx_N, \vy_N) \in \R^d \times \R^T$ and a fixed width\footnote{By an argument similar to the proof of Theorem 5 of \cite{shenouda2024variation}, as long as $K \geq N^2$, problem \eqref{opt:pn} is well-posed and attains the same minimal objective value (regardless of which $K \geq N^2$ is chosen). Therefore, in this work, we always assume that $K$ is some fixed value larger than $N^2$.} \label{fn:K_N_squared} $K \geq N^2$, we consider the weight decay interpolation problem:

\begin{equation} \label{opt:wd}
 \min_{\btheta}  \sum_{k=1}^{K} \|\vv_k\|_{2}^{2} + \|\vw_k\|_{2}^{2} \ , \ \mbox{subject to } f_{\btheta}(\vx_i)=\vy_i, \, i=1,\dots,N \ .
\end{equation}
By homogeneity of the ReLU activation function (meaning that $(\alpha x)_+ = \alpha (x)_+$ for any $\alpha \geq 0$), the input and output weights of any ReLU neural network can be rescaled as $\vw_k \mapsto \vw_k / \| \vw_k \|_2$ and $\vv_k \mapsto \vv_k \| \vw_k \|_2$ without changing the function that the network represents. Using this fact, several previous works \cite{grandvalet1998least,grandvalet1998outcomes,NeyshaburInductiveBias,parhi2023deep} note that problem \eqref{opt:wd} is equivalent to
\begin{eqnarray}
\label{opt:pn}
 \min_{\btheta} 
  \sum_{k=1}^{K} \|\vv_k\|_{2}, \  \mbox{subject to }  \{\|\vw_k\|_2=1\}_{k=1}^K, \,   f_\btheta(\vx_i)=\vy_i, \, i=1,\dots,N \ 
\end{eqnarray}
in that the minimal objective values of both training problems are the same, and any network $f_{\btheta}$ which solves \eqref{opt:wd} also solves \eqref{opt:pn}, while any $f_{\btheta}$ which solves \eqref{opt:pn} also solves \eqref{opt:wd} after rescaling of the input and output weights. The regularizer $\sum_{k=1}^K \| \vv_k \|_2$ is reminiscent of the multi-task lasso \cite{obozinski2006multi}. It has recently been shown to promote \emph{neuron sharing} in the network, such that only a few neurons contribute to all tasks \cite{shenouda2024variation}.

The optimizations in \eqref{opt:wd} and \eqref{opt:pn} are non-convex and in general, they may have multiple global minimizers.  As an example, consider the single-task, univariate dataset in \cref{fig:uni}. For this dataset, \eqref{opt:pn} has infinitely many global solutions \cite{savarese2019infinite,ergen2021convex,debarre2022sparsest,hanin2021ridgeless}. Two of the global minimizers are shown in \cref{fig:uni}. In some scenarios, the solution on the right may be preferable to the one on the left, since the interpolation function stays closer to the training data points, and hence is more adversarially robust by most definitions \cite{carlini2019evaluating}. Moreover, recent theoretical work shows that this solution has other favorable generalization and robustness properties \cite{joshi2024noisy}.  Current training methods, however, might produce any one of the infinite number of solutions, depending on the random initialization of neural network weights as well as other possible sources of randomness in the training process.   It is impossible to control this using existing training algorithms, which might explain many problems associated with current neural networks such as their sensitivity to adversarial attacks. In contrast, as we show in this paper, training a network to interpolate the data in \cref{fig:uni} along with additional interpolation tasks with different labels almost always produces a unique solution, given by the (potentially preferable) interpolation depicted on the right. This demonstrates that the solutions to multi-task learning problems can be profoundly different than those of single-task learning problems.

\begin{figure}[h]
    \centering
    {\includegraphics[width=12cm]{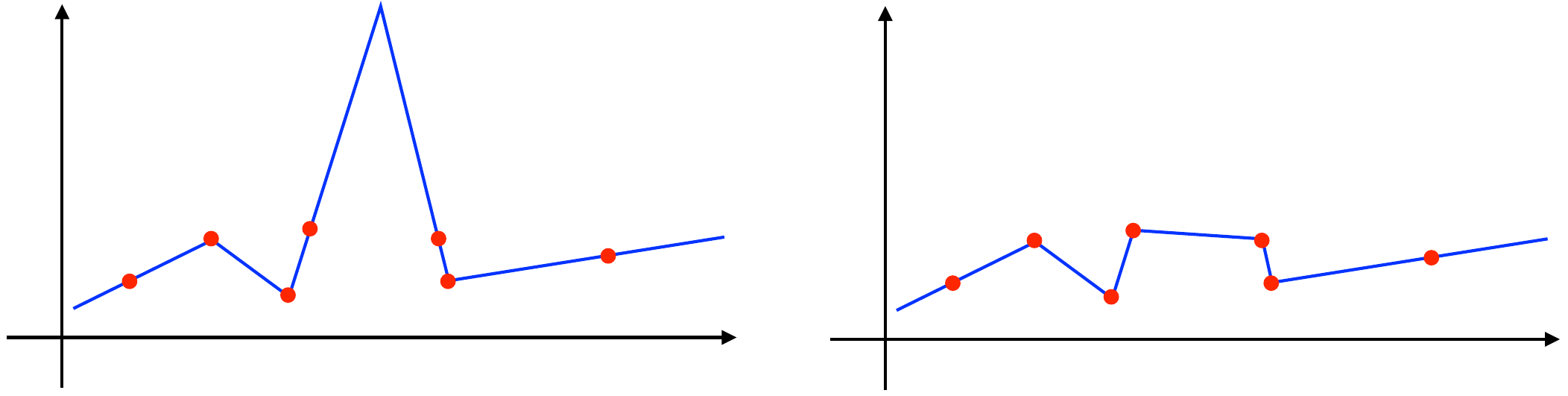}} 
    \caption{Two solutions to ReLU neural network interpolation (blue) of training data (red).  The functions on the left and right both interpolate the data and both are global minimizers of \eqref{opt:wd} and \eqref{opt:pn}, and minimize the second-order total variation of the interpolation function \cite{parhi2021banach}. In fact, all convex combinations of the two solutions above are also solutions to this learning problem.}
    \label{fig:uni}
\end{figure}

The main contributions of our paper are:
\begin{enumerate}[wide, labelwidth=!, labelindent=0pt]
    \item[\bf Uniqueness of Multi-Task Solutions.] In the univariate setting ($d=1$) we prove that the solution to multi-task learning problems with different tasks almost always represent a unique function, and we give a precise condition for the exceptional cases where solutions are non-unique. 
    \item[\bf Multi-Task Learning $\equiv$ Kernel Method (almost always).] When the solution to the univariate weight decay problem is unique, it is given by the connect-the-dots interpolant of the training data points: i.e., the optimal solution is a linear spline which performs straight-line interpolation between consecutive data points in all tasks.  On the support of the data, this solution agrees with the minimum-norm interpolant in the first-order Sobolev space $H^1$, a reproducing kernel Hilbert space (RKHS) which contains all functions with first derivatives in $L^2$ \cite{de1966splines}.  In contrast, solutions to the single-task learning problem are non-unique in general and are given by minimum-norm interpolating functions in the non-Hilbertian Banach space $\BV^2$ \cite{parhi2021banach}, which contains all functions with second distributional derivatives\footnote{Technically, $\BV^2$ contains all functions with second distributional derivatives in $\cM$, the space of Radon measures with finite total variation. $\cM$ can be viewed as a ``generalization'' of $L^1$ (see Ch. 7.3, p.223 in \cite{folland1999real}).} in $L^1$. This shows that the individual task solutions to a multi-task learning problem are almost always equivalent to those of a minimum-norm kernel interpolation problem, whereas single-task solutions generally are not. 
    \item[\bf Insights on Multivariate Multi-Task Learning.] We provide empirical evidence and mathematical analysis which indicate that similar conclusions hold in multivariate settings. Specifically, the individual task solutions to a multi-task learning problem are approximately minimum-norm solutions in a particular RKHS determined by the optimal neurons. In contrast, learning each task in isolation results in solutions that are minimum-norm with respect to a non-Hilbertian Banach norm over the optimal neurons.
\end{enumerate}
\section{Related Works}
\paragraph{\textbf{Characterizations of ReLU neural network solutions: }} 
\cite{hanin2021ridgeless,stewart2023regression} characterized the neural network solutions to \eqref{opt:wd} in the univariate input/output setting. \cite{boursier2023penalising} showed that in the univariate input/output case, when weight decay is modified to include the biases of each neuron, the solution is unique. Moreover, under certain assumptions, it is the sparsest interpolant (i.e., the interpolant with the fewest neurons). Our work differs from these in that we study the multi-task setting, showing that univariate-input multi-task solutions are almost always unique and equivalent to the connect-the-dots solution, which is generally \textit{not} the sparsest, and is a minimum-norm solution in a Sobolev RKHS.
While characterizing solutions to \eqref{opt:wd} in the multivariate setting is more challenging, there exist some results for very particular datasets \cite{ergen2021convex, ardeshir2023intrinsic, zeno2024minimum}.

\paragraph{\textbf{Function spaces associated with neural networks:}}
For single-output ReLU neural networks, \cite{savarese2019infinite, ongie2019function}
related weight decay regularization on the parameters of the model to regularizing a particular semi-norm on the neural network function. \cite{ongie2019function} showed that this semi-norm is not an RKHS semi-norm, highlighting a fundamental difference between learning with neural networks and kernel methods. \cite{parhi2021banach,parhi2022kinds, bartolucci2023understanding, UnserUnifyingRepresenter} studied the function spaces associated with this semi-norm, and developed representer theorems showing that optimal solutions to the minimum-norm data fitting problem over these spaces are realized by finite-width ReLU networks. Consequently, finite-width ReLU networks trained with weight decay are optimal solutions to the regularized data-fitting problem posed over these spaces. Function spaces and representer theorems for multi-output and deep neural networks were later developed in \cite{korolev2022two, parhi2022kinds, shenouda2024variation}.

\paragraph{\textbf{Multi-Task Learning: }}
 The advantages of multi-task learning have been extensively studied in the machine learning literature \cite{obozinski2006multi, obozinski2010joint, argyriou2006multi, argyriou2008convex,caruana1997multitask}. In particular, the theoretical properties of multi-task neural networks have been studied in \cite{lindsey2023implicit,collins2024provable, shenouda2024variation}. The underlying intuition in these past works has been that learning multiple related tasks simultaneously can help select or learn the most useful features for all tasks. Our work differs from this traditional paradigm as we consider multi-task neural networks trained on very general tasks which may be diverse and unrelated.

\section{Univariate Multi-Task Neural Network Solutions} \label{sec:univariate}
For any function $f$ that can be represented by a neural network \eqref{eq:nn} with width $K$, we define its representational cost to be
\begin{eqnarray}\label{eq:rep_cost}
    R(f) := \inf_{\bm{\theta}} \sum_{k=1}^{K} \|\vv_k\|_{2} \ , \  \text{subject to } \|\vw_k\|_2 = 1 \ \forall k , \, f= f_{\vtheta}
\end{eqnarray}
where $\vtheta = \left(\{\vv_k, \vw_k, b_k\}^{K}_{k=1}, \mA, \vc\right)$. Taking an inf over all possible neural network parameters is necessary as there are multiple neural networks which can represent the same function.
Solutions to \eqref{opt:pn} minimize this representational cost subject to the data interpolation constraint. This section gives a precise characterization of the solutions to the multi-task neural network interpolation problem in the univariate setting ($d=1$).

For the training data points $(x_1, \vy_1), \dots, (x_N, \vy_N) \in \R \times \R^T$, let $y_{it}$ denote the $t^{\text{th}}$ coordinate of the label vector $\vy_i$. For each $t = 1, \dots, T$, let $\cD_t$ denote the univariate dataset $(x_1, y_{it}), \dots, (x_N, y_{Nt}) \in \R \times \R$, and let
\begin{equation}
    s_{it} = \frac{y_{i+1,t} - y_{it}}{x_{i+1} - x_{i}}
\end{equation}
denote the slope of the straight line between $(x_i, y_{it})$ and $(x_{i+1}, y_{i+1 t})$.
The connect-the-dots interpolant of the dataset $\cD_t$  is the function $f_{\cD_t}$ which connects the consecutive points in dataset $\cD_t$ with straight lines (see \cref{fig:ctd_interp}). Its slopes on $(-\infty, x_2]$ and $[x_{N-1}, \infty)$ are $s_{1t}$ and $s_{N-1 t}$, respectively. In the following section, we state a simple necessary and sufficient condition under which the connect-the-dots interpolation $f_{\cD} = (f_{\cD_1}, \dots, f_{\cD_T})$ is the \emph{unique} optimal interpolant of the datasets $\cD_1, \dots, \cD_T$. We also demonstrate that the set of multi-task datasets which satisfy the necessary condition for non-uniqueness, viewed as a subset of $\R^N \times \R^{T \times N}$, has Lebesgue measure zero. 

This result raises an interesting new connection between data fitting with ReLU neural networks and traditional kernel-based learning methods. Indeed, connect-the-dots interpolation is also the minimum-norm interpolant over the first-order Sobolev space $H^1([x_1, x_N])$, itself an RKHS whose norm penalizes the $L^2$ norm of the derivative of the function. In particular, $f_{\cD_t}$ agrees on $[x_1, x_N]$ with the function $f(x) = \sum_{j=1}^N \alpha_j k(x, x_j)$ whose coefficients $\alpha_j$ solve the kernel optimization problem 
\begin{equation}
\min_{\alpha_1,\dots,\alpha_N \in \R} \sum_{i=1}^N\sum_{j=1}^N \alpha_i\alpha_j  k(x_i,x_j) \ , \ \mbox{subject to } \sum_{j=1}^N \alpha_j k(x_i,x_j) = y_{it}, \, i=1,\dots,N \ .
\end{equation}
with the kernel $k(x, x') = 1-(x-x')_+ + (x-x_1)_+ + (x_1 - x')_+$ \cite{de1966splines}. Therefore, our result shows that the individual outputs of solutions to \eqref{opt:pn} for $T > 1$ tasks almost always coincide on $[x_1, x_N]$ with this kernel solution. In contrast, optimal solutions to the \eqref{opt:pn} in the case $T = 1$ are generally non-unique and may not coincide with the connect-the-dots kernel solution \cite{hanin2022implicit, debarre2022sparsest}. We note that for $T = 1$, our result is consistent with the characterizations of univariate solutions to \eqref{opt:pn} in \cite{hanin2022implicit, debarre2022sparsest}.

\begin{figure}
    \centering
    \includegraphics[width=.6\textwidth]{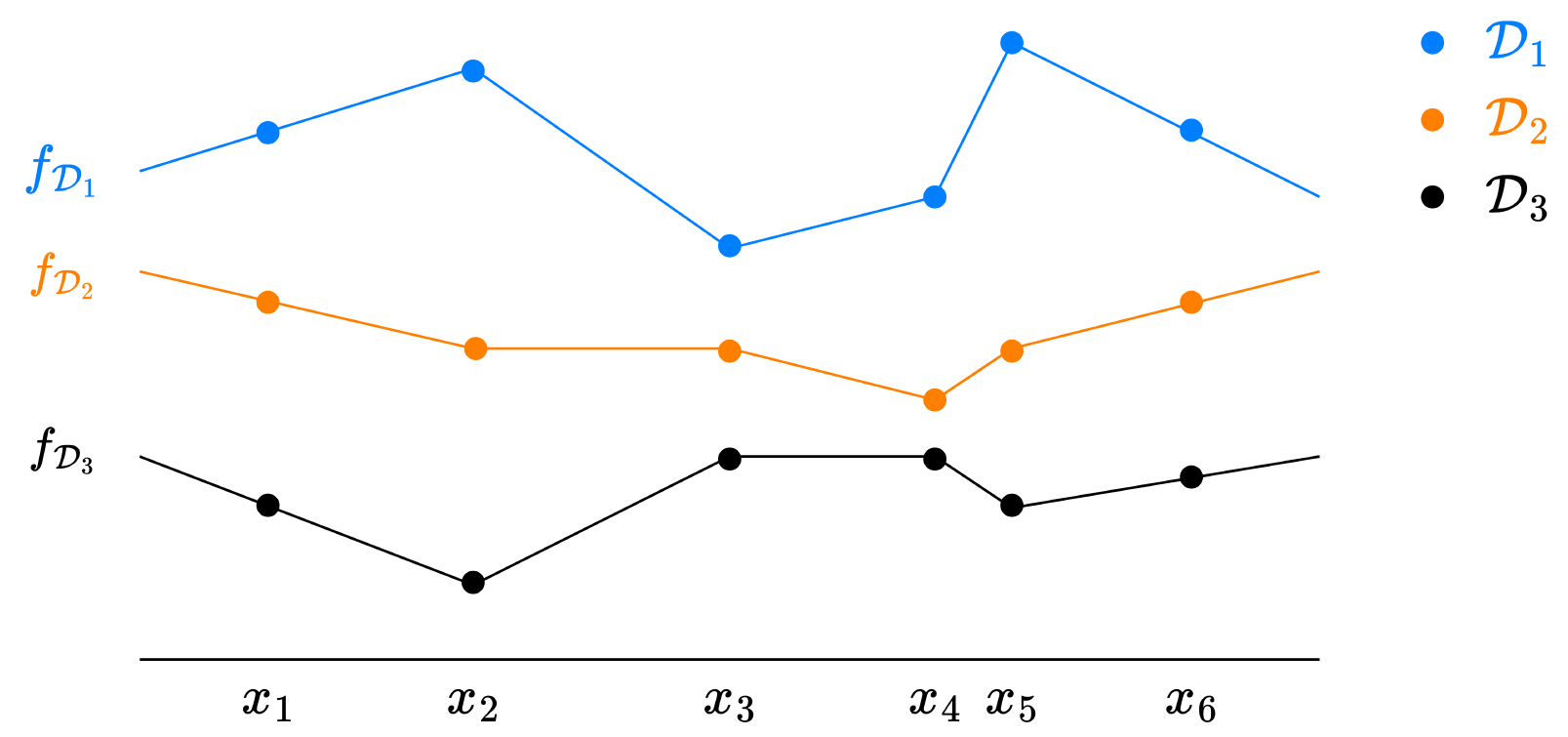}
    \caption{The connect-the-dots interpolant $f_{\cD} = (f_{\cD_1}, f_{\cD_2}, f_{\cD_3})$ of three datasets $\cD_1, \cD_2, \cD_3$.}
    \label{fig:ctd_interp}
\end{figure}

\subsection{Characterization and Uniqueness}

Our main result is stated in the following theorem:

\begin{theorem}\label{th:univariate_ctd_main_theorem}
The connect-the-dots function $f_{\cD}$ is always a solution to \eqref{opt:pn}. Moreover, the solution to problem \eqref{opt:pn} is non-unique if and only if the following condition is satisfied: for some $i = 2, \dots, N-2$, the two vectors
\begin{align}
    \vs_i - \vs_{i-1} = \frac{\vy_{i+1} - \vy_i}{x_{i+1}-x_i} - \frac{\vy_i - \vy_{i-1}}{x_i - x_{i-1}}
\end{align}
and
\begin{align}
    \vs_{i+1} - \vs_i = \frac{\vy_{i+2}-\vy_{i+1}}{x_{i+2}-x_{i+1}} - \frac{\vy_{i+1} - \vy_i}{x_{i+1}-x_i}
\end{align}
are both nonzero and aligned.\footnote{Two vectors $\vu_1$ and $\vu_2$ are \emph{aligned} if $\vu_1^{\top}\vu_2 = \|\vu_1\|\|\vu_2\|$. \label{fn:aligned_def}} 

If this condition not satisfied, then $f_{\cD}$ is the unique solution to \eqref{opt:pn}. Furthermore, as long as $T > 1$ and $N > 1$, the set of all possible data points $x_1, \dots, x_N \in \R$ and $\vy_1, \dots, \vy_N \in \R^T$ which admit non-unique solutions has Lebesgue measure zero (as a subset of $\R^N \times \R^{T \times N}$).
\end{theorem}
\begin{corollary}
    If $T > 1$ and $N >1$ and the data points $x_1, \dots, x_N \in \R$ and label vectors $\vy_1, \dots, \vy_N \in \R^T$ are sampled from an absolutely continuous distribution with respect to the Lebesgue measure on $\R^N \times \R^{T \times N}$, then with probability one, the connect-the-dots function $f_{\cD}$ is the unique solution to \eqref{opt:pn}.
\end{corollary}
\begin{remark}
    The proof of \cref{th:univariate_ctd_main_theorem}, which relies mainly on \cref{lemma:keylemma} as we describe below, also characterizes solutions of the regularized loss problem
    \begin{eqnarray} \label{eq:reg_loss}
         \min_{\btheta} \sum_{i=1}^{N} \mathcal{L}(f_{\btheta}(\vx_i), \vy_i) + \lambda \sum_{k=1}^{K} \|\vv_k\|_2 \quad  \mbox{subject to} \quad |w_k|=1, \: k=1,\dots, K  
    \end{eqnarray}
    for input dimension $d = 1$, any $\lambda > 0$, and any loss function $\cL$ which is lower semicontinuous in its second argument. Specifically, any $f_{\vtheta}$ which solves \eqref{eq:reg_loss} is linear between consecutive data points $[x_i, x_{i+1}]$ unless the vectors $\hat{\vs}_i - \hat{\vs}_{i-1}$ and $\hat{\vs}_{i+1} - \hat{\vs}_i$ are both nonzero and aligned, where $\hat{\vs}_i := \frac{\hat{\vy}_{i+1}-\hat{\vy}_i}{x_{i+1}-x_i}$ and $\hat{\vy}_i := f_{\vtheta}(x_i)$. 
\end{remark}

Previous works \cite{shenouda2024variation} and \cite{lindsey2023implicit} showed that multi-task learning encourages \textit{neuron sharing}, where all task are encouraged to utilize the same set of neurons or representations. Our result above shows that univariate multi-task training is an extreme example of this phenomenon, since $f_{\cD}$ can be represented using only $N-1$ neurons, all of which contribute to all of the network outputs. Therefore, in the scenario we study here, neuron sharing almost always occurs even if the tasks are unrelated.

The full proof of \cref{th:univariate_ctd_main_theorem} appears in \cref{appendix:proof_univariate_ctd_main_theorem}. We outline the main ideas here. Our proof relies on the fact that any $\R \to \R^T$ ReLU neural network of the form \eqref{eq:nn} which solves \eqref{opt:pn} represents $T$ continuous piecewise linear (CPWL) functions, where the change in slope of the $t^\textrm{th}$ function at the $k^{\text{th}}$ knot is equivalent the $t^\textrm{th}$ entry of the $k^{\text{th}}$ output weight vector (see \cref{appendix:proof_univariate_ctd_main_theorem} for further detail). 
This fact allows us to draw a one-to-one correspondence between each knot in the function and each neuron in the neural network. The proof relies primarily on the following lemma:
 
\begin{lemma} \label{lemma:keylemma}
Let $f: \R \to \R^T$ be a function for which each output $f_t$ is CPWL and interpolates $\cD_t$. Suppose that at some $\tilde{x}$ between consecutive data points, one or more of the outputs $f_t$ has a knot. Let $\tilde{x}_1$  and $\tilde{x}_2$ be the $x$-coordinates of the closest knots before and after $\tilde{x}$, respectively. Denote the slopes of $f_t$ around this interval in terms of $a_t$, $b_t$, $c_t$, and $\delta_t$ as in \cref{fig:generic_example}, and let $\va, \vb, \vc, \vdelta \in \R^T$ be the vectors containing the respective values for each task.
 
Then removing the knot at $x$ and instead connecting $x_i$ and $x_{i+1}$ by a straight line does not increase $R(f)$. Furthermore, if $\va - \vb$ and $\vb - \vc$ are \emph{not} aligned, then doing so will strictly decrease $R(f)$.
\end{lemma}
\begin{figure}[h]
    \begin{center}
    \includegraphics[width=0.6\linewidth]{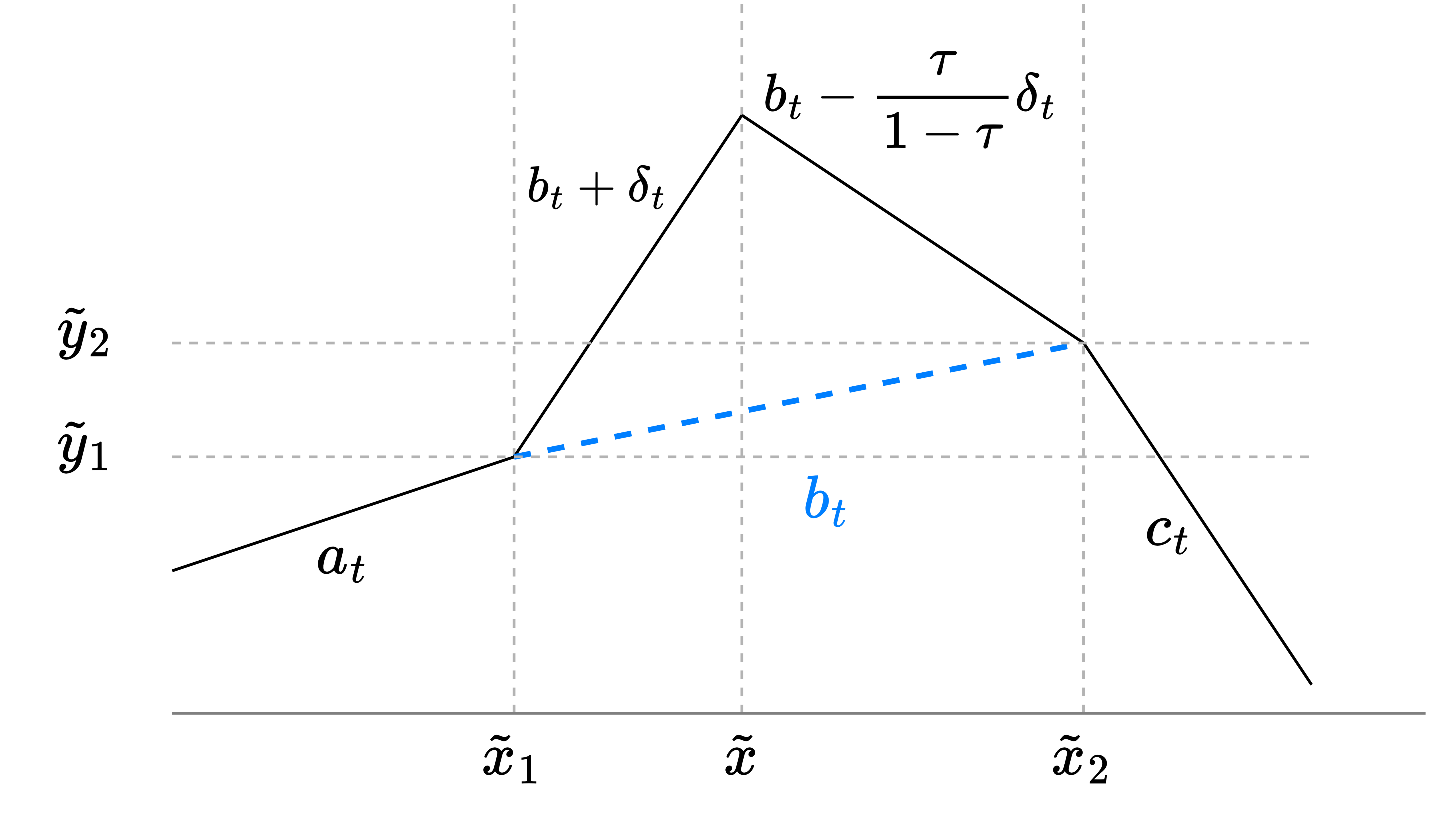}
    \caption{The function output $f_t$ around the knot at $\tilde{x}$, where $\tau = \frac{\tilde{x}-\tilde{x}_1}{\tilde{x}_2-\tilde{x}_1}$. Each line segment in the figure is labeled with its slope. For any particular output $t$, it may be the case that $f_t$ does not have a knot at $\tilde{x}$ (in which case $\delta_t = 0$); does not have a knot at $\tilde{x}_1$ (in which case $a_t = b_t + \delta_t$); and/or does not have a knot at $\tilde{x}_2$ (in which case $b_t - \frac{\tau}{1-\tau} \delta_t = c_t$).}
    \label{fig:generic_example}
    \end{center}
\end{figure}
\begin{proof}[Proof of \cref{lemma:keylemma}]
When represented by a ReLU neural network, each knot in the function corresponds to a neuron. The representational cost  \eqref{eq:rep_cost} is separable across knots/neurons. The contribution of these knots to $R(f)$ is:
\begin{eqnarray}
     & & \hspace{-.5in} \|\vdelta+\vb-\va\|_2 \ + \ \frac{1}{1-\tau} \|\vdelta\|_2 \ + \ \|\vc-\vb+\tau\vdelta/(1-\tau)\|_2 \nonumber \\
    &\geq &\|\vb-\va\|_2-\|\vdelta\|_2 \ + \ \frac{1}{1-\tau} \|\vdelta\|_2 \ + \ \|\vc-\vb\|_2 -\frac{\tau}{1-\tau}\|\vdelta\|_2 \label{eq:rte} \\
    &= &\|\vb-\va\|_2 \ + \ \|\vc-\vb\|_2 \  \nonumber
\end{eqnarray}
where the inequality follows from the triangle inequality.
This shows that taking $\delta_t = 0$ for all outputs, which corresponds to connecting $\tilde{x}_1$ and $\tilde{x}_2$ with a straight line in all outputs, will never increase the representational cost of $f$. The triangle inequality used in \eqref{eq:rte} holds with equality for some $\vdelta \neq \bm{0}$ if and only if $\va - \vb$, $\vb - \vc$, and $\vdelta$ are aligned with $\| \vdelta \|_2 \leq \min \{ \| \va - \vb \|_2, \frac{1-\tau}{\tau} \| \vb - \vc \|_2 \}$.
\end{proof}

\cref{lemma:keylemma} states that removing neurons which are located away from the data points and replacing them with a straight line will never increase the cost of the network, and it will strictly decrease the cost unless $\va-\vb$ and $\vb-\vc$ are aligned. This result implies that the connect-the-dots interpolant $f_{\cD}$ is always a solution to \eqref{opt:pn}, since we may take any solution $f$ of \eqref{opt:pn} and remove all knots from it (resulting in the function $f_{\cD}$) without increasing its representational cost. If $\vs_i - \vs_{i-1}$ and $\vs_{i+1}-\vs_i$ are aligned for some $i =2, \dots, N-2$, we can view any interpolant on the interval $[x_i, x_{i+1}]$ as an instance of \cref{fig:generic_example} with $\va = \vs_{i-1}$, $\vb = \vs_i$, and $\vc = \vs_{i+1}$. By \cref{lemma:keylemma}, any CPWL function with a knot at some point $\tilde{x} \in (x_i, x_{i+1})$ can have the same representational cost as the connect-the-dots solution on this interval, only if $\va - \vb$ and $\vb - \vc$ are aligned. 

We can also prove by contradiction that optimal solutions are unique on $[x_i, x_{i+1}]$ as long as $\vs_i - \vs_{i-1}$ and $\vs_{i+1}- \vs_i$ are \emph{not} aligned. Suppose that there is some other optimal interpolant $f$ which is \emph{not} the connect-the-dots solution $f_{\cD}$ on an interval $[x_i, x_{i+1}]$ for which $\vs_i - \vs_{i-1}$ and $\vs_{i+1}- \vs_i$ are not aligned. Then apply the lemma repeatedly to remove all knots from $f_\vtheta$ which are not located at the data points, except for a single remaining knot at some $\tilde{x}$ between consecutive data points. If this knot occurs after $x_2$ (the second data point) or before $x_{N-1}$ (second to last data point), the lemma implies automatically (again taking $\va = \vs_{i-1}$, $\vb = \vs_i$, and $\vs_{i+1}$) that removing this knot would strictly decrease the representational cost of the function, contradicting optimality of $f$. To conclude the proof, it remains only to show that any optimal interpolant of the dataset must agree with the connect-the-dots interpolant $f_{\cD}$ before $x_2$ and after $x_{N-1}$; the details of this argument appear in \cref{appendix:proof_univariate_ctd_main_theorem}.  As our theorem and corollary quantify, real-world regression datasets (which are typically real-valued and often assumed to incorporate some random noise from an absolutely continuous distribution, e.g. Gaussian) are extremely unlikely to satisfy this special alignment condition; hence, our claim that connect-the-dots interpolation is almost always the unique solution to \eqref{opt:pn}.

\subsection{Numerical Illustration of \cref{th:univariate_ctd_main_theorem}}
We provide numerical examples to illustrate the difference in solutions obtained from single task versus multi-task training and validate our theorem. The first row in \cref{fig:uni_experiments} shows three randomly initialized univariate neural networks trained to interpolate the five red points with minimum sum of squared weights. 
While all three of the learned functions have the same representational cost (i.e., all minimize the second-order total variation subject to the interpolation constraint), they each learn different interpolants. This demonstrates that gradient descent does not induce a bias towards a particular solution. 
The second row shows the function learned for the first output of a multi-task neural network. This network was trained on two tasks. The first task consists of interpolating the five red points while the second consists of interpolating five randomly generated labels sampled from a standard Gaussian distribution. When trained to interpolate with minimum sum of squared weights we see that the connect-the-dots solution is the only solution learned regardless of initialization, verifying \cref{th:univariate_ctd_main_theorem}.  This solution simultaneously minimizes the second-order total variation and the norm in the first-order Sobolev RKHS $H^1([x_1, x_N])$ associated with the kernel $k(x, x') = 1-(x-x')_+ + (x-x_1)_+ + (x_1 - x')_+$ \cite{de1966splines}, subject to the interpolation constraints.\footnote{The code for all numerical experiments can be found by clicking on the Jupyter logo \jupyter{https://github.com/joeshenouda/effects-mtl-nns} in the caption, which links to the notebook that reproduces that figure.}

\begin{figure}[!h]
\begin{subfigure}{0.3\linewidth}
    \includegraphics[width=\linewidth]{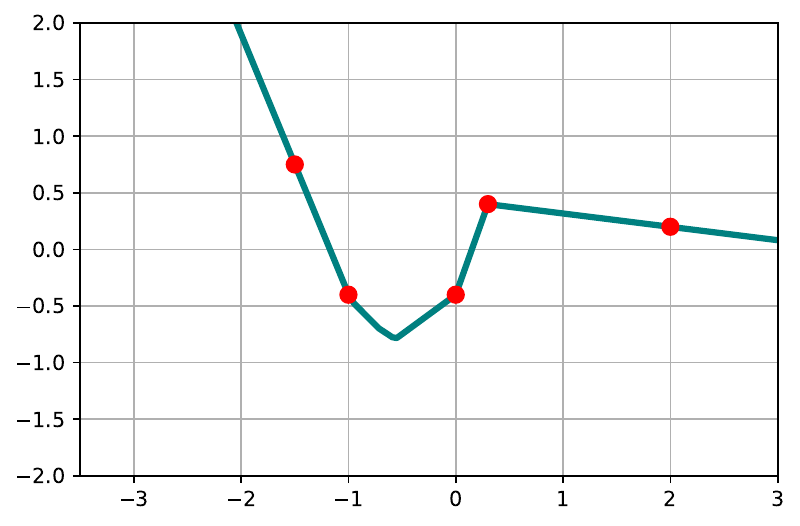}
\end{subfigure}
\hfill
\begin{subfigure}{0.3\linewidth}
    \includegraphics[width=\linewidth]{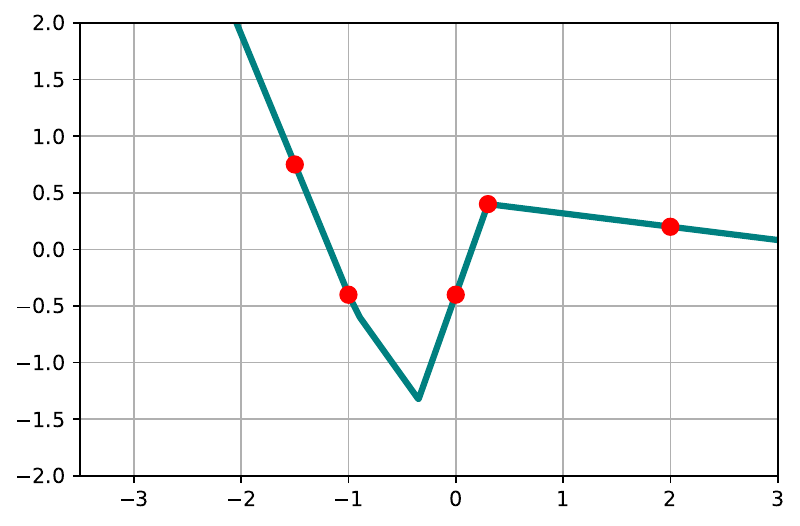}
\end{subfigure}
\hfill
\begin{subfigure}{0.3\linewidth}
    \includegraphics[width=\linewidth]{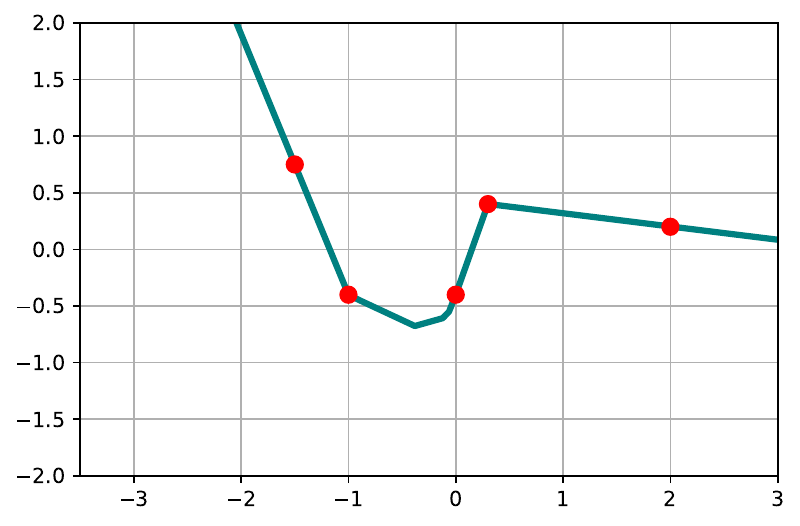}
\end{subfigure}

\bigskip

\begin{subfigure}{0.3\linewidth}
    \includegraphics[width=\linewidth]{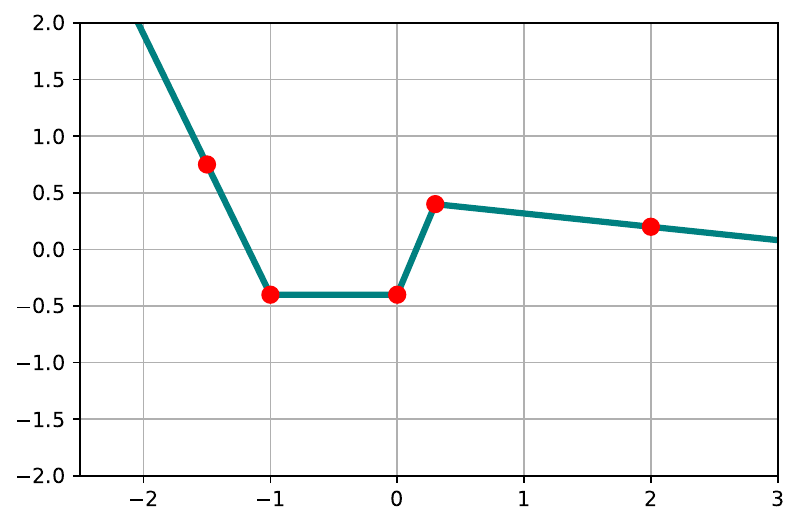}
\end{subfigure}
\hfill
\begin{subfigure}{0.3\linewidth}
    \includegraphics[width=\linewidth]{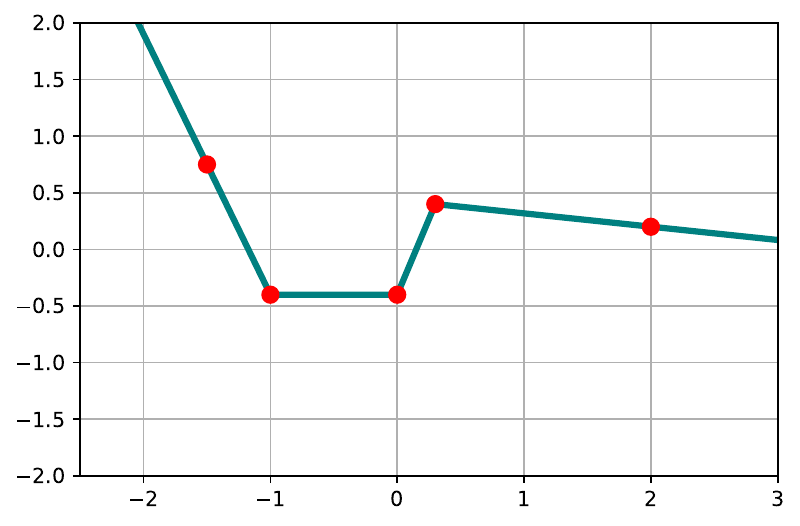}
\end{subfigure}
\hfill
\begin{subfigure}{0.3\linewidth}
    \includegraphics[width=\linewidth]{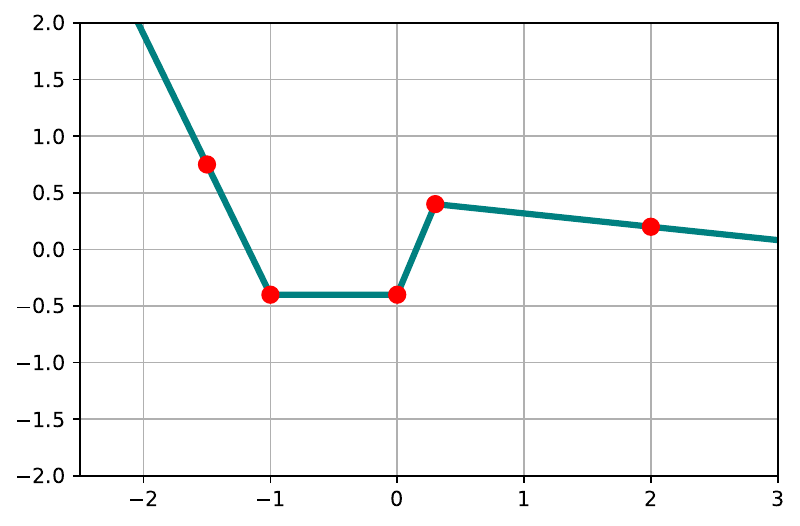}
\end{subfigure}
\caption{\textit{Top Row}: Three randomly initialized neural networks trained to interpolate the five red points with minimum sum of squared weights. \textit{Bottom Row}: Three randomly initialized two-output neural networks trained to interpolate a multi-task dataset with minimum sum of squared weights. The labels for the first task are the five red points shown while the labels for the second were randomly sampled from a standard Gaussian distribution. \jupyter{https://github.com/joeshenouda/effects-mtl-nns}}
\label{fig:uni_experiments}
\end{figure}
\newpage

\section{Multivariate Multi-Task Neural Network Training} \label{sec:multivariate} 
In \cref{sec:univariate}, we proved that the univariate-input functions learned by neural networks trained on multiple tasks simultaeously can be profoundly different from the functions learned by networks trained for each task separately.  In this section, we demonstrate that a similar phenomenon occurs in multivariate settings. Here we analyze neural networks of the form
\begin{equation} \label{eq:nn_no_skip}
    f_{\vtheta}(\vx) = \sum_{k=1}^K \vv_k \big( \vw_k^\top \vx + b_k \big)_+
\end{equation}
where $\vw_k \in \bS^{d-1}$, $b_k \in \R$, $\vv_k \in \R^T$, and $\vtheta := \{ \vv_k, \vw_k, b_k \}_{k=1}^K$. Since the analysis in this section is not dependent on the residual connection, we omit it for ease of exposition. We consider the multivariate-input, $T$-task neural network training problem 
\begin{equation}\label{opt:multivariate_problem_T_tasks}
     \min_{\btheta}  \sum_{i=1}^N \cL \left( \vy_{i}, f_{\vtheta}(\vx_i) \right) + \lambda \sum_{k=1}^K \| \vv_k \|_2 
\end{equation} 
for some dataset $(\vx_1, \vy_{1}), \dots, (\vx_N, \vy_{N}) \in \R^d \times \R^T$, where $\cL$ is any loss function which is lower semicontinuous in its second argument and separable across the $T$ tasks\footnote{In other words, $\cL$ is expressible as $\cL(\vu, \vv) = \sum_{t=1}^T \tilde{\cL}(u_t, v_t)$ for any $\vu, \vv \in \mathbb{R}^T$, for some univariate loss function $\tilde{\cL}$. For notational convenience, we denote both the multivariate and univariate loss functions as $\cL$.}, and $K \geq N^2$ (see footnote \footref{fn:K_N_squared}).

We are interested in analyzing the behavior of solutions to \eqref{opt:multivariate_problem_T_tasks} as the number of tasks $T$ grows. Intuitively, if $T$ is very large, it
it is reasonable to expect that the optimal output weight $v_{ks}^*$ for an individual neuron $k$ and task $s$ would be relatively small compared to the sum of the output weights $v_{kt}^*$ for tasks $t \neq s$. In this case, the $k^\textrm{th}$ term of the regularizer in \eqref{opt:multivariate_problem_T_tasks} would be approximately equal to
\begin{align} \label{eq:taylor_approx_intuition}
    \| \vv_k^* \|_2 \ = \ \sqrt{(v_{ks}^*)^2 + \| \vv_{k \setminus s}^* \|_2^2 } \ \approx \ \| \vv_{k \setminus s}^* \|_2 + \frac{(v_{ks}^*)^2}{2 \| \vv_{k \setminus s}^* \|_2} 
\end{align}
for any individual task $s$, where $\| \vv_{k \setminus s}^* \|_2^2 := \sum_{t \neq s} (v_{ks}^*)^2$. The approximation above comes from the Taylor expansion $f(x) = \sqrt{x^2 + c^2} = c + \frac{x^2}{2c} - \frac{x^4}{8 c^3} + \frac{x^6}{16c^5} - \dots$, whose higher order terms quickly become negligible if $0 < x \ll c$. Notice that the right hand side of \eqref{eq:taylor_approx_intuition} is a quadratic function of $v_{ks}^*$, which suggests that the regularization term of \eqref{opt:multivariate_problem_T_tasks} resembles a weighted $\ell^2$ regularizer when $v_{ks}$ is close to its optimal value $v_{ks}^*$. 

The above reasoning can be made precise using the concept of \textit{exchangeability}. A sequence of random variables is exchangeable if their joint distribution remains unchanged for any finite permutation of the sequence. Intuitively, exchangeability captures the notion that the order of the random variables does not matter. This is naturally consistent with multi-task learning problems, in which the order in which the tasks are enumerated and assigned to the outputs of the network is irrelevant.  Therefore, the $T$ tasks can be assigned uniformly at random to the $T$ outputs of the network. This random assignment process induces a distribution on the training data (labels) for each output, denoted by $\vy_{\cdot,1}, \dots, \vy_{\cdot,T} \in \R^N$. These vectors are comprised of the $N$ labels for each task corresponding to the inputs $\vx_1,\dots,\vx_N$.
In particular, $\vy_{\cdot,1}, \dots, \vy_{\cdot, T}$ are exchangeable. In the remainder of the discussion, we assume that the vectors $\vy_{\cdot,1}, \dots, \vy_{\cdot,T}$ have been generated according to this procedure. 

Using the concept of task exchangeability, we consider the problem of minimizing
\begin{align} \label{eq:J}
    J(v_{1s}, \dots, v_{Ks}) :=  \sum_{i=1}^N \cL \left(   y_{is}, \sum_{k=1}^K v_{ks} \mPhi_{ik}   \right) + \lambda \sum_{k=1}^K \left\| \begin{bmatrix}
        v_{ks} \\ \vv_{k \setminus s}^*
    \end{bmatrix} \right\|_2
\end{align}
over $\bR^K$, where $s$ is an individual task. $J$ is simply the objective function of \eqref{opt:multivariate_problem_T_tasks} with all parameters except for $v_{1s}, \dots, v_{Ks}$ held fixed at their optimal values, and with the (constant) data fitting terms from the other tasks except $s$ removed. Loosely speaking, we can view $J$ as the objective function of \eqref{opt:multivariate_problem_T_tasks} from the perspective of a single task $s$. Note that the optimal values $v_{1s}^*, \dots, v_{Ks}^*$ for \eqref{opt:multivariate_problem_T_tasks} also minimize $J$; otherwise they would not be optimal for \eqref{opt:multivariate_problem_T_tasks}. The following theorem shows that, when the task labels are exchangeable, the regularization term in $J$ is approximately quadratic near its minimizer and that this approximation get stronger as the number of tasks $T$ increases.
\begin{theorem} \label{th:main_multivariate}
    Let $\mathcal{S}$ be the set of optimal active neurons\footnote{A neuron $\eta(\vx) = \vv (\vw^{T} \vx + b)_+$  is \emph{active} if $\|\vv\|_2 > 0 $.} solving \eqref{opt:multivariate_problem_T_tasks}, that is, the neurons for which $\|\vv^{*}_k\|_2>0$. For an individual task $s$, consider the objective
    \begin{align} \label{eq:H}
        H(v_{1s}, \dots, v_{Ks}) :=  \sum_{i=1}^N  \cL \left(   y_{is}, \sum_{k \in \mathcal{S}} v_{ks} \mPhi_{ik}   \right) + \lambda \sum_{k \in \mathcal{S}} \left( \| \bm{v}_{k \setminus s}^* \|_2 + \frac{v_{ks}^2}{2 \| \bm{v}_{k \setminus s}^* \|_2} \right)
    \end{align}
    where $\vv_{k \setminus s}^*$ denotes the vector $\vv_k^*$ with its $s^{\text{th}}$ element $v_{ks}^*$ excluded, and $\mPhi \in \R^{N \times K}$  is a matrix whose $i,k^{\text{th}}$ entry is $\mPhi_{ik} = \big( \vx_i^\top \vw_k^*  + b_k^* \big)_+$. Then with probability at least $1-O(T^{-1/3})$, the regularization term in $H$ is well-defined, and
    \begin{align} \label{eq:J_H_bound}
        |J(v_{1s}, \dots, v_{Ks}) - H(v_{1s}, \dots, v_{Ks})| \ = \   O(T^{-1/4})
    \end{align}
    whenever $|v_{ks}| \leq T^{1/16} |v_{ks}^*|$ for all $k = 1, \dots, K$. As a consequence, the global minimizer $v_{1s}', \dots, v_{Ks}'$ of $H$ satisfies
    \begin{align} \label{eq:J_v_prime_v_star}
        J(v_{1s}', \dots, v_{Ks}') - J(v_{1s}^*, \dots, v_{Ks}^*) \ = \  O(T^{-1/4})
    \end{align}
    with probability at least $ 1-O \left( T^{-1/3} \right)$.
\end{theorem}

The proof of \cref{th:main_multivariate} is in \cref{appendix:proof_multivariate_main_theorem}. The theorem states that the solution to the followed weighted $\ell^2$ regularized problem
\begin{align} \label{opt:weighted_l2}
    \min_{v_{1s}, \dots, v_{Ks}} \sum_{i=1}^N  \cL \left(   y_{is}, \sum_{k \in \mathcal{S}} v_{ks} \mPhi_{ik}   \right) + \frac{\lambda}{2} \sum_{k \in \mathcal{S}} \gamma_{ks} v_{ks}^2 
\end{align}
where $\gamma_{ks} := 1/\|\vv^{*}_{k/s}\|_2$ is approximately optimal for the original objective \eqref{eq:J}, with stronger approximation as $T$ increases. In contrast, when $T = 1$, the optimization
\begin{align} \label{opt:l1}
     \min_{v_{1}, \dots, v_{K}} \sum_{i=1}^N \cL \left(   y_{i}, \sum_{k=1}^K v_{k} \mathbf{\Psi}_{ik}   \right) + \lambda \sum_{k=1}^K |v_{k}|
\end{align}
yields output weights which are exactly optimal for \eqref{opt:multivariate_problem_T_tasks}. Note that the matrices $\mPhi$ in \eqref{opt:weighted_l2} and $\mathbf{\Psi}$ in \eqref{opt:l1} are not the same, since they are determined by the optimal input weights and biases for \eqref{opt:multivariate_problem_T_tasks}, which are themselves data- and task-dependent. Nonetheless, comparing \eqref{opt:weighted_l2} and \eqref{opt:l1} highlights the different nature of solutions learned for \eqref{opt:multivariate_problem_T_tasks} in the single-task versus multi-task case. The multi-task learning problem with exchangeable tasks favors linear combinations of the optimal neurons which have a minimal weighted $\ell^2$ regularization penalty. In contrast, the single-task learning problem favors linear combinations of optimal neurons which have a minimal $\ell^1$ penalty. Therefore, multi-task learning with a large number of tasks promotes a fundamentally different linear combination of the optimal features learned in the hidden layer. 

Additionally, exchangeability of the tasks allows for the following characterization of the coefficients $\gamma_{ks}$ in \eqref{opt:weighted_l2} (cf \cref{lemma:subvectors} in \cref{appendix:proof_multivariate_main_theorem}):
\begin{lemma} \label{lemma:subvectors_gamma}
For any $s = 1, \dots, T$, any $k = 1, \dots, K$, and any $0 < \beta < 1$:
\begin{align}
    \left(1-T^{\beta-1} \right) \| \vv_k^* \|_2^2 \leq \| \vv_{k \setminus s}^* \|_2^2 \leq \| \vv_k^* \|_2^2
\end{align}
with probability at least $1-T^{-\beta}$.
\end{lemma}
The proof is in \cref{appendix:proof_multivariate_main_theorem}. This result suggests that the weights in the weighted $\ell^2$ regularizer for each task converge to a common value across all the tasks with high probability as $T$ grows larger. In particular, for large $T$:
\begin{align}
    \gamma_{ks} \approx \gamma_{k} = \frac{1}{\|\vv^{*}_{k}\|_2}
\end{align}
with high probability. 
This reveals a novel connection between the problem of minimizing \eqref{eq:J} and a norm-regularized data fitting problem in an RKHS. Specifically, consider the finite-dimensional linear space
\begin{align}
    \cH := \left\{ f_\vv = \sum_{k=1}^K v_k \phi_k: \vv \in \R^K \right\}
\end{align}
where $\phi_k(\vx) = \left(\vw_k^* \cdot \vx + b_k^* \right)_+$, equipped with the inner product
\begin{align}
    \langle f_\vv, f_\vu \rangle_\cH = \vv^\top \mQ \vu
\end{align}
where $\mQ = \textrm{diag} \left( \frac{\gamma_1}{2}, \dots, \frac{\gamma_K}{2}  \right)$. As a finite-dimensional inner product space, $\cH$ is necessarily a Hilbert space; furthermore, finite-dimensionality of $\cH$ implies that all linear functionals (including the point evaluation functional) on $\cH$ are continuous. Therefore, $\cH$ is an RKHS, with reproducing kernel
\begin{align}
    \kappa (\vx, \vx') = \sum_{k=1}^K \phi_k (\vx) Q_{kk}^{-1} \, \phi_k (\vx').
\end{align}
Note that $\kappa$ indeed satisfies the reproducing property, that is, $\langle \kappa(\cdot, \vx), f\rangle_{\mathcal{H}} = f(\vx)$ for any $f \in \mathcal{H}$ and any $\vx$. To see this, write
\begin{align}
    \langle \kappa(\cdot, \vx), f\rangle_{\mathcal{H}} =  \left\langle \sum_{k=1}^{K} Q^{-1}_{kk} \phi_k(\vx) \phi_k, \sum_{k=1}^{K} v_k \phi_k \right\rangle_{\mathcal{H}}.
\end{align}
We can view the term on the left as a function $g_{\vu} \in \mathcal{H}$ where $u_k = Q_{kk}^{-1} \phi_k(\vx)$ or $\vu = \phi(\vx) \mQ^{-1}$, so this is equivalent to
\begin{align}
    \langle \kappa(\cdot, \vx), f\rangle_{\mathcal{H}} =  \left\langle g_{\vu}, f \right\rangle_{\mathcal{H}} = \phi(\vx) \mQ^{-1} \mQ \vv = \sum_{k=1}^{K} v_k \phi(\vx) = f(\vx).
\end{align}
Finding a minimizer of $H$ over $\R^K$ is thus equivalent to solving
\begin{align} \label{opt:RKHS_problem}
    \argmin_{f \in \cH} \sum_{i=1}^N \cL (y_{is}, f(\vx_i)) + \lambda \| f \|_\cH^2.
\end{align}
We provide empirical evidence for the claims presented in this section in \Cref{fig:bi} on a simple multi-variate dataset. First, we demonstrate the variety of solutions that interpolate this dataset in a single task setting. In contrast, we show that the solutions obtained via multi-task learning with additional random tasks are very similar and often much smoother than those obtained by single-task learning supporting our claim that these solutions are well approximated by solving a kernel ridge regression problem. We also verify that the optimization \eqref{eq:H} is a good approximation for \eqref{eq:J}.  We include additional experiments in \cref{appendix:additional_experiments} that demonstrate that these observations hold across multiple trials.
\begin{figure} 
    \centering
    \begin{subfigure}{0.3\textwidth}
        \includegraphics[width=\textwidth]{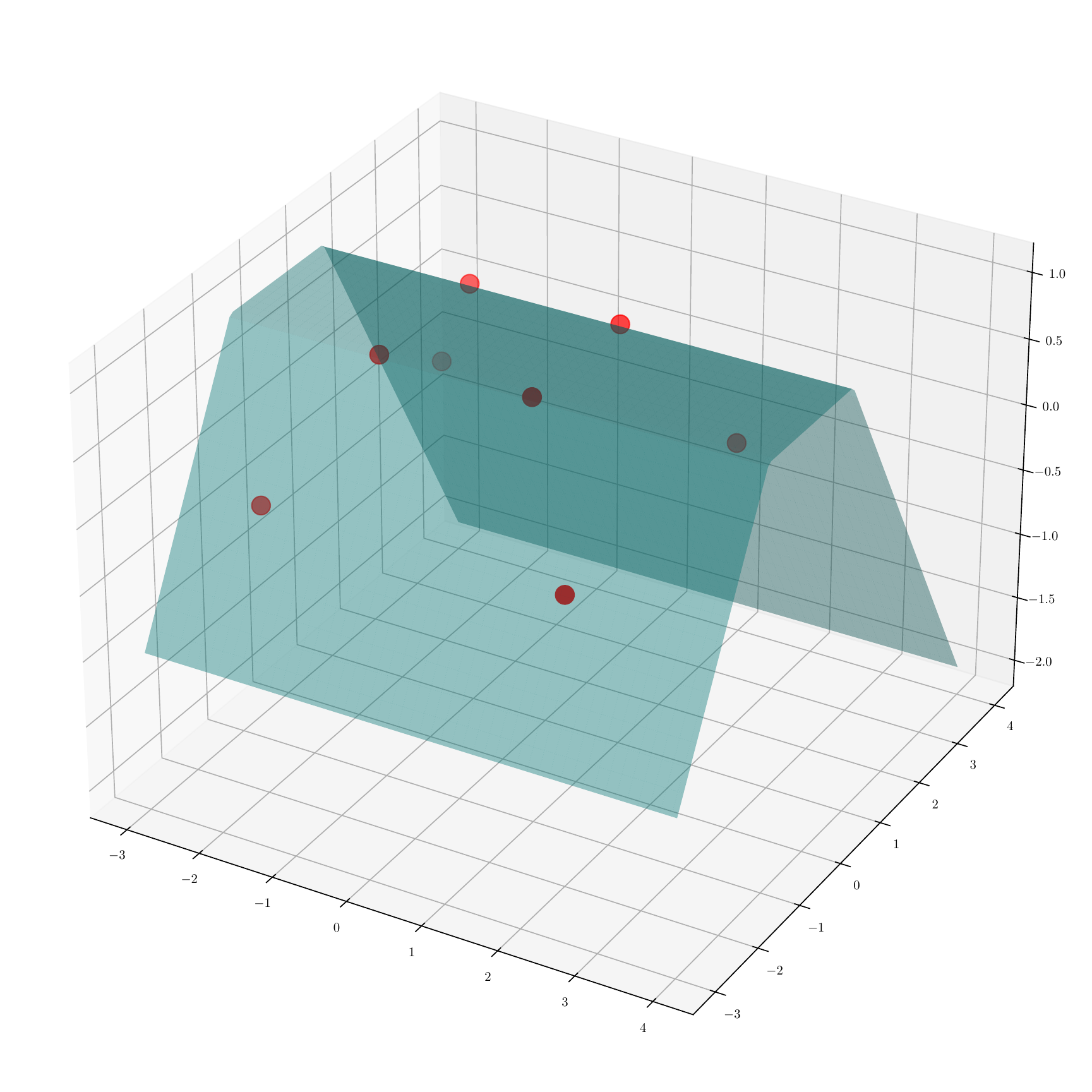}
        \caption{Solution to \eqref{opt:multivariate_problem_T_tasks} with $T=1$ (single-task training)}
        \label{fig:sing_sol_1}
    \end{subfigure}
    \hfill
    \begin{subfigure}{0.3\textwidth}
        \includegraphics[width=\textwidth]{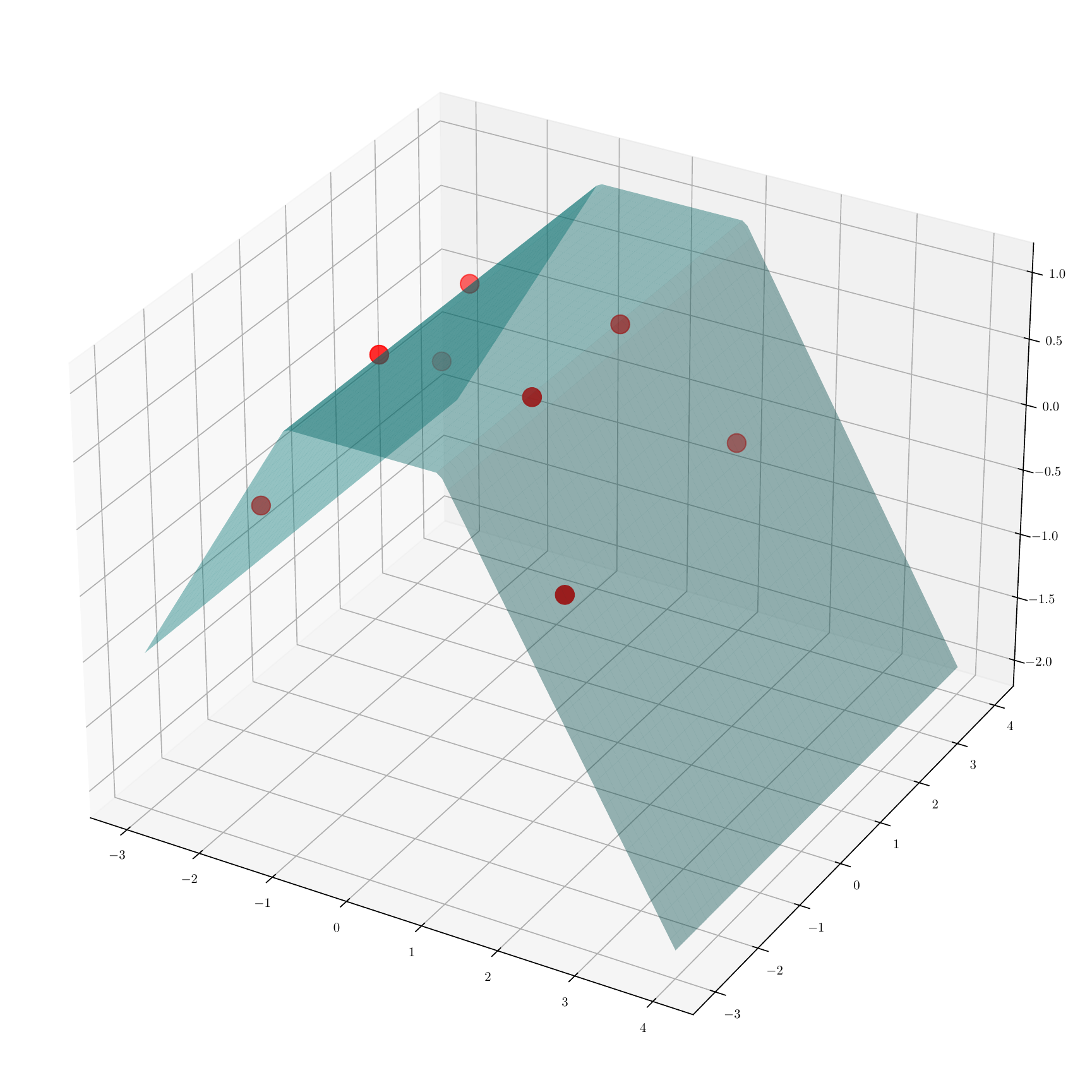}
        \caption{Solution to \eqref{opt:multivariate_problem_T_tasks} with $T=1$ (single-task training)}
        \label{fig:sing_sol_2}
    \end{subfigure}
    \hfill
    \begin{subfigure}{0.3\textwidth}
        \includegraphics[width=\textwidth]{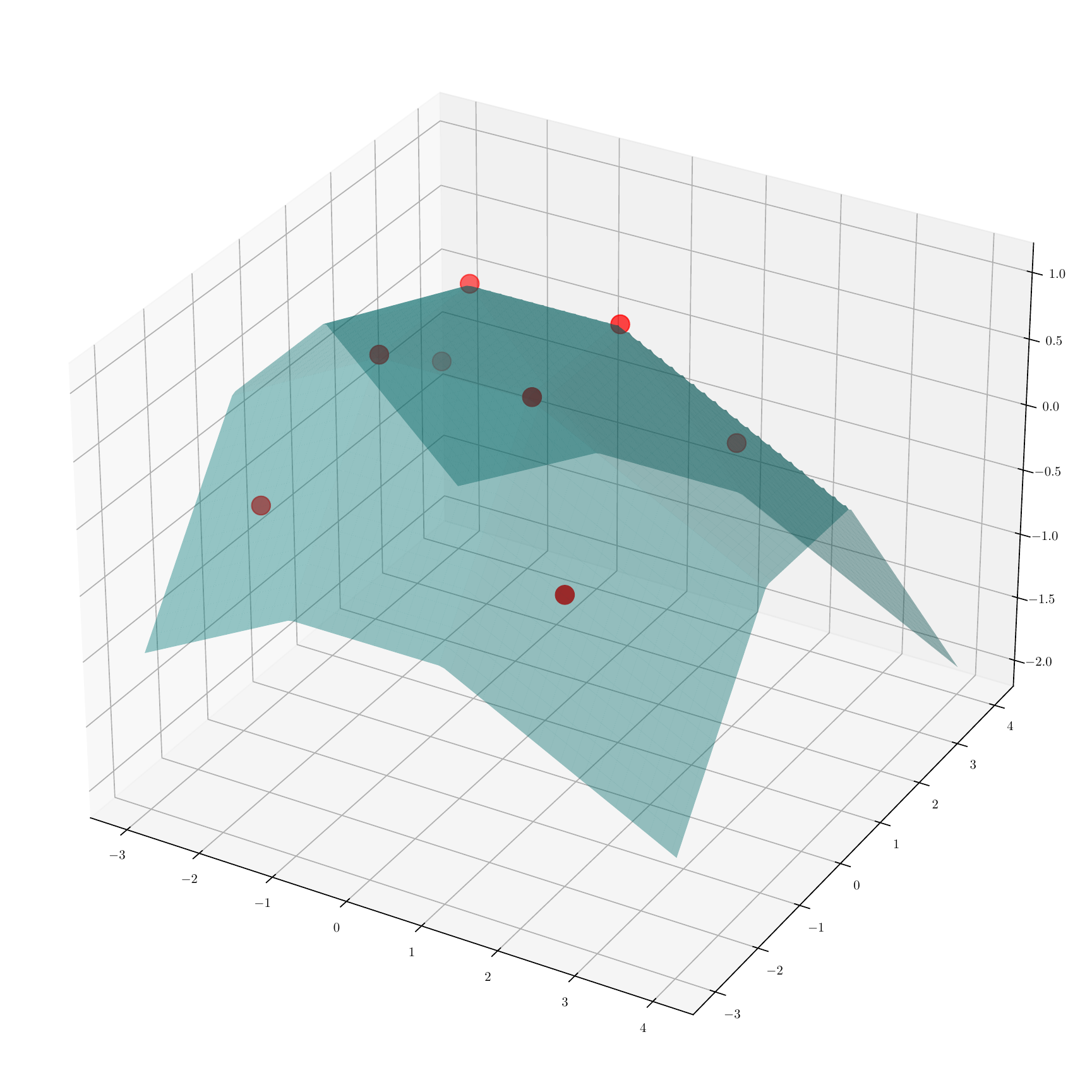}
        \caption{Solution to \eqref{opt:multivariate_problem_T_tasks} with $T=1$ (single-task training)}
        \label{fig:sing_sol_3}
    \end{subfigure}
    \centering
    \bigskip
    \hfill
    \begin{subfigure}{0.45\textwidth}
        \includegraphics[width=\textwidth]{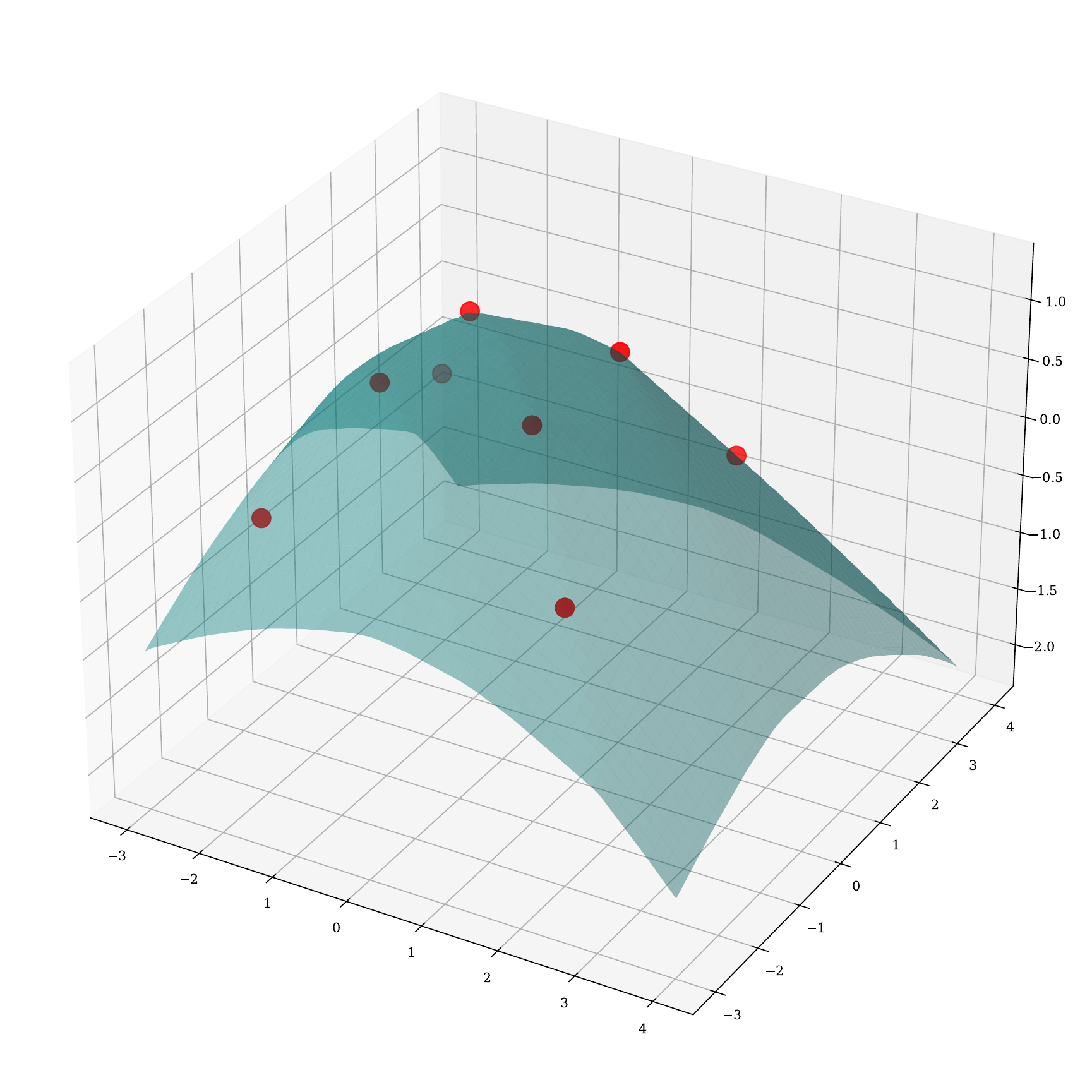}
        \caption{Solution to \eqref{opt:multivariate_problem_T_tasks} with $T=101$\\ (multi-task training)}
        \label{fig:mtl_sol}
    \end{subfigure}
    \hfill
    \begin{subfigure}{0.45\textwidth}
        \includegraphics[width=\textwidth]{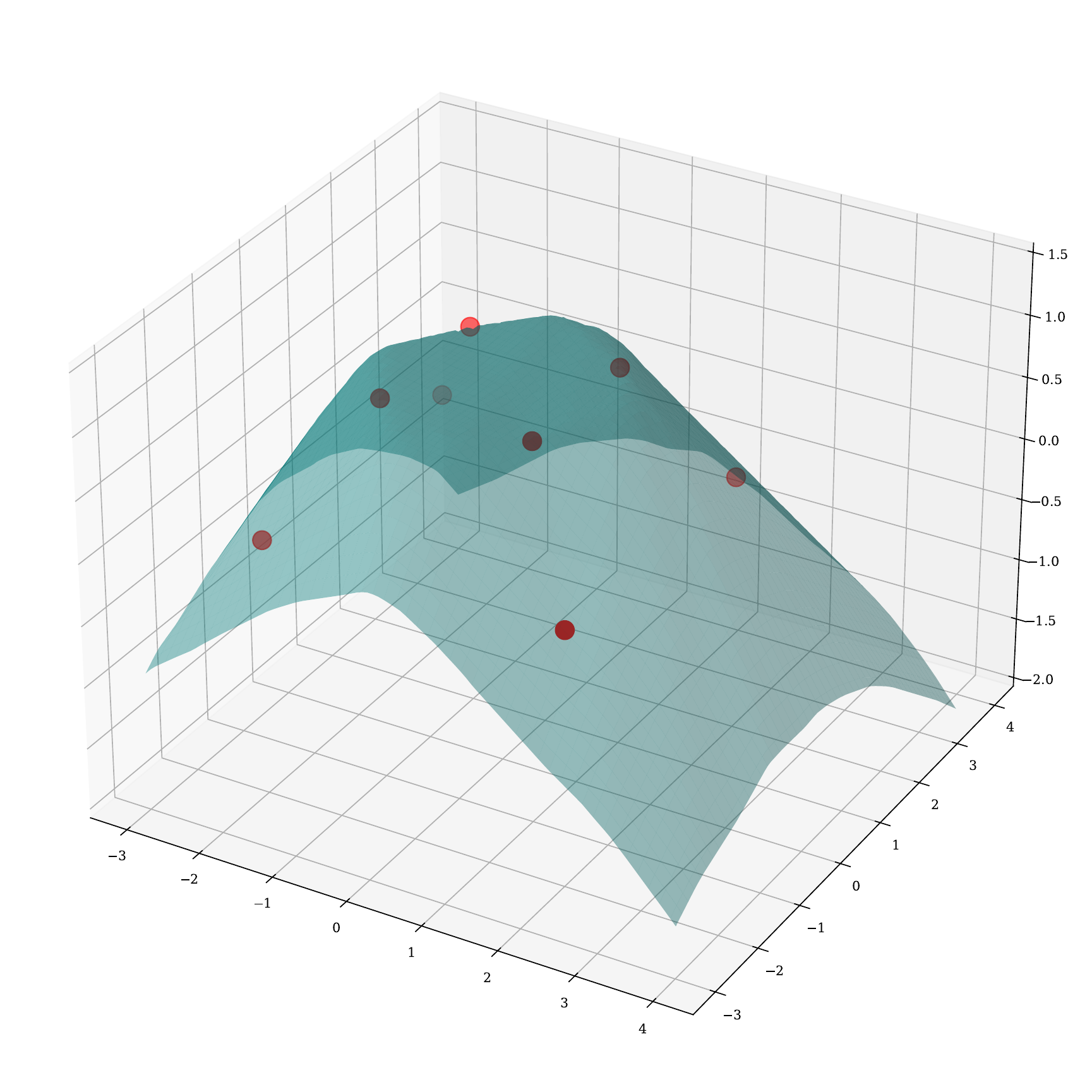}
        \caption{Solution to \eqref{opt:RKHS_problem} \\(RKHS approximation)}
        \label{fig:rkhs_sol}
    \end{subfigure}
    \hfill
    \caption{ReLU network interpolation in two-dimensions.  The solutions shown were obtained with regularization parameter $\lambda \approx 0$. \textit{Top Row -- Solutions to single-task training}: \Cref{fig:sing_sol_1,,fig:sing_sol_2,,fig:sing_sol_3} show solutions to ReLU neural network interpolation (blue surface) of training data (red). The eight data points are located at the vertices of two squares, both centered at the origin. The outer square has side-length two and values of $0$ at the vertices.  The inner square has side-length one and values of $1$ at the vertices.   All three functions interpolate the data and are global minimizers of \eqref{opt:wd} and \eqref{opt:pn} when solving for just this task (i.e., $T=1$). Due to the simplicity of this dataset the optimality of the solutions in the first row were confirmed by solving the equivalent convex optimization to \eqref{opt:wd} developed in \cite{ergen2021convex}. \textit{Bottom Row -- Solutions to multi-task training:} \Cref{fig:mtl_sol} shows the solution to the first output of a multi-task neural network with $T=101$ tasks. The first output is the original task depicted in the first row while the labels for other $100$ tasks are randomly generated i.i.d from a Bernoulli distribution with equal probability for one and zero.  
    Here we show one representative example; more examples are depicted in \cref{appendix:additional_experiments} showing that this phenomenon holds across many runs. \Cref{fig:rkhs_sol} shows the solution to fitting the training data by solving \eqref{opt:RKHS_problem} over a fixed set of features learned by the multi-task neural network with $T=100$ random tasks. 
    We observe that unlike the highly variable solutions of single-task optimization problem, the solutions obtained by solving the multi-task optimizations are nearly identical, as one would have for kernel methods. Moreover, the solution obtained by solving \eqref{opt:RKHS_problem} is also similar to the solution of the full multi-task training problem with all $T=101$ tasks. \jupyter{https://github.com/joeshenouda/effects-mtl-nns}
    } 
    \label{fig:bi}
\end{figure}

\section{Conclusion and Discussion} \label{sec:conclusion_limitations_future_work}
We have shown that univariate, multi-task shallow ReLU neural networks which are trained to interpolate a dataset with minimal sum of squared weights almost always represent a unique function. This function performs straight-line interpolation between consecutive data points for each task. This solution is also the solution to a min-norm data-fitting problem in an RKHS. We provide mathematical analysis and numerical evidence suggesting that a similar conclusion may hold in the mulvariate-input case, as long as the tasks are sufficiently large in number. These results indicate that multi-task training of neural networks can produce solutions that are strikingly different from  those obtained by single-task training, and highlights a novel connection between these multi-task solutions and kernel methods.

 Future work could aim to extend these results to deep neural network architectures. We also focus here on characterizing global solutions to the optimizations in \eqref{opt:wd} and \eqref{opt:pn}. Whether or not networks trained with gradient descent-based algorithms will converge to global solutions remains an open question: our low-dimensional numerical experiments in \cref{sec:univariate,,sec:multivariate} indicate that they do, but a more rigorous analysis of the training dynamics would be an interesting separate line of research. Finally, while our analysis and experiments in \cref{sec:multivariate} indicate that multi-variate, multi-task neural network solutions behave similarly to $\ell^2$ regression over a fixed kernel, we have not precisely characterized what that kernel is in the multi-input case as we have in the single-input case: developing such a characterization is of interest for future work.

\newpage
\bibliography{refs}
\bibliographystyle{abbrvnat}
\newpage

\begin{appendices}
\section{Proofs of Main Results}
\subsection{Proof of \cref{th:univariate_ctd_main_theorem}} \label{appendix:proof_univariate_ctd_main_theorem}
\begin{proof}  
We break the proof into the following sections.

\paragraph{Unregularized Residual Connection.} 
We first discuss the utility of the unregularized residual connection in our analysis. Consider a single-input/output function \( f \) represented by a ReLU network \( f_{\vtheta}: \R \to \R \) with unit norm input weights. Suppose that \( f_{\vtheta} \) includes two neurons, 
$$\eta_1(x) = v_1 (w_1 x + b_1)_+ \quad \text{ and }\quad  \eta_2(x) = v_2 (w_2 x + b_2)_+ $$
with $ b_1/w_1 = b_2/w_2 $, i.e., $\eta_1$ and $\eta_2$ both ``activate'' at the same location. There are two possible cases:

\begin{enumerate}
    \item \textbf{Same Sign Input Weights:} If $ \sgn(w_1) = \sgn(w_2) $, the neurons activate in the same direction and can be merged into a single neuron with input weight $w_1 + w_2$, bias $b_1 + b_2$, and output weight $v_1 + v_2$, without altering the representational cost or the represented function.
    \item \textbf{Opposite Sign Input Weights:} If $ w_1 = 1 $ and $ w_2 = -1 $, their sum can be rewritten using $ x = (x)_+ - (-x)_+ $:
    $$
    v_1 (x + b_1)_+ + v_2 (-x - b_1)_+ = (v_1 + v_2)(x+b_1)_+ - v_2(x+b_1).
    $$
    The term $  -v_2(x+b_1) $ can be absorbed into the residual connection, representing the same function with no greater representational cost.
\end{enumerate}
The residual connection allows us to conclude that for optimal networks solving \eqref{opt:pn}, no two neurons activate at the same location. This provides a one-to-one correspondence between slope changes in the function and neurons in the network. The same conclusion also holds for any $T$ output ReLU network. 

Therefore, any set of $T$ continuous piecewise linear (CPWL) functions with a combined total of $K$ slope changes (i.e., knots) at locations $\tilde{x}_1,\dots, \tilde{x}_K$ is represented, with minimal representational cost, by a network of width $K$, where parameters satisfy $ -b_k/w_k = \tilde{x}_k $ and $ \mu_{kt} = w_k v_{kt} $ (here $\mu_{kt}$ denotes the slope change of the $t^\textrm{th}$ function at $\tilde{x}_k$). This correspondence enables us to analyze networks which solve \eqref{opt:pn} entirely using CPWL functions, treating ``knots'' and ``neurons'' interchangeably and using $ |v_{kt}| $ to denote both the magnitude of the slope change at a knot and the magnitude of the output weight.

\paragraph{Connect-the-dots interpolation is always a solution to \eqref{opt:pn}.} 
Using \cref{lemma:keylemma}, we proceed to prove \cref{th:univariate_ctd_main_theorem}. First note that the objective function in \eqref{opt:pn} is coercive and continuous, and if $K \geq N$, the feasible set of \eqref{opt:pn} is non-empty and---as the preimage of a closed set under the neural network function (which is continuous with respect to its parameters)---closed. Therefore, a solution to \eqref{opt:pn} exists \cite[Lemma 2.14]{beck2017first}. Let $S_{\btheta}^*$ denote the set of parameters of optimal neural networks which solve \eqref{opt:pn} for the given data points, and let
\begin{align}
    S^* := \{ f: \R \to \R^T \ \rvert \  f(x) = f_{\btheta}(x) \, \; \forall x \in \R, \; \btheta \in S_{\btheta}^* \}
\end{align}
be the set of functions represented by neural networks with optimal parameters in $S_{\btheta}^*$. First, note that the connect-the-dots interpolant $f_{\cD}$ is in the solution set $S^*$. To see this, fix any $f \in S^*$, and apply \cref{lemma:keylemma} repeatedly to remove all ``extraneous'' knots (i.e., knots located away from the data points $x_1, \dots, x_N$) from $f$. By \cref{lemma:keylemma}, the resulting function---which is simply $f_{\cD}$---has representational cost no greater than the original $f$, and since $f$ had optimal representational cost, so does $f_{\cD}$. 

\newpage
\paragraph{Condition for non-unique solutions.}
We first illustrate the condition under which \eqref{opt:pn} admits an infinite number of solutions.
For some $i = 2, \dots, N-2$, consider the two vectors
\begin{align}
    \vs_i - \vs_{i-1} = \frac{\vy_{i+1} - \vy_i}{x_{i+1}-x_i} - \frac{\vy_i - \vy_{i-1}}{x_i - x_{i-1}}
\end{align}
and
\begin{align}
    \vs_{i+1} - \vs_i = \frac{\vy_{i+2}-\vy_{i+1}}{x_{i+2}-x_{i+1}} - \frac{\vy_{i+1} - \vy_i}{x_{i+1}-x_i}.
\end{align}
and assume they are aligned. Then we may view the function $f_{\cD}$ around the interval $[x_i, x_{i+1}]$ as an instance of \cref{lemma:keylemma} with $\va = \vs_{i-1}$, $\vb = \vs_i$, and $\vc = \vs_{i+1}$. Fix some point $\tilde{x} \in (x_i, x_{i+1})$ and denote 
$$\tau = \frac{\tilde{x} - x_{i}}{x_{i+1}-x_i}.$$ 
Let $\vdelta$ be any vector which is aligned with $\va-\vb$ and $\frac{1-\tau}{\tau}(\vb - \vc)$ and has smaller norm than both, and let $f: \R \to \R^T$ be the function whose output slopes on $(x_i, \tilde{x})$ are given by $\vdelta$ and whose slopes on $(\tilde{x}, x_{i+1})$ are given by $\vb - \frac{\tau}{1-\tau} \vdelta$. Then by \cref{lemma:keylemma}, $R(f) = R(f_{\cD})$ and thus $f \in S^*$. Since there are infinitely many such $\vdelta$'s, there are infinitely many optimal solutions to \eqref{opt:pn}. 

\paragraph{Necessary and sufficient condition under which $f_{\cD}$ is the unique solution.}
For the other direction of the proof, suppose that the vectors $\vs_i - \vs_{i-1}$ and $\vs_{i+1} - \vs_i$ are \emph{not} aligned for any $i = 1, \dots, N-1$, and assume by contradiction that there is some $f \in S^*$ which is \emph{not} of the form $f_{\cD}$. This $f$ must not have any knots on $(-\infty,x_1]$ or $[x_N, \infty)$, since removing such a knot would strictly decrease $R(f)$ without affecting the ability of $f$ to interpolate the data, contradicting optimality of $f$. So it must be the case that $f$ has an extraneous knot at some $\tilde{x}$ which lies between consecutive data points $x_i$ and $x_{i+1}$. Let $g$ denote the function obtained by removing all extraneous knots from $f$ \emph{except} the one located at $\tilde{x}$. By \cref{lemma:keylemma}, $R(g) \leq R(f)$.

First consider the case where the remaining extraneous knot is between $[x_i, x_{i+1}]$ for $i = 2, \dots, N-2$. Since $g$ has no extraneous knots away from $\tilde{x}$, it must be the case that $g$ agrees with $f_{\cD}$ on $[x_{i-1}, x_i]$ and $[x_{i+1}, x_{i+2}]$. We may view the behavior of $g$ around the interval $[x_i, x_{i+1}]$ as an instance of \cref{lemma:keylemma} with $\va = \vs_{i-1}$, $\vb = \vs_i$, and $\vc = \vs_{i+1}$. By assumption, $\va - \vb$ and $\vb - \vc$ are \emph{not} aligned, so by \cref{lemma:keylemma}, removing the knot at $\tilde{x}$ would strictly reduce $R(g)$. This contradicts optimality of $g$, hence of $f$.

\begin{figure}
    \begin{subfigure}{.5\textwidth}
          \centering
          \includegraphics[width=.9\linewidth]{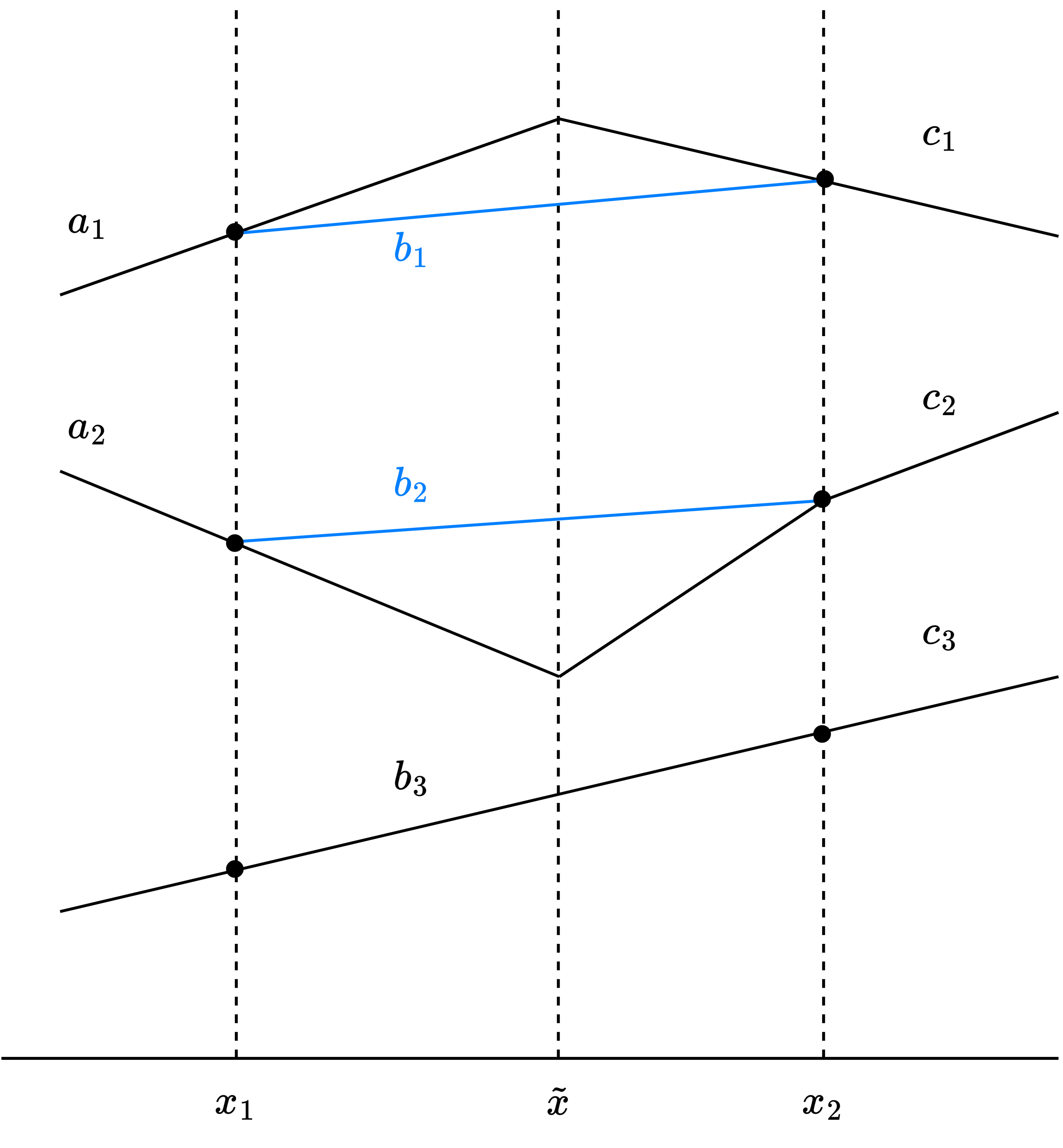}
          \caption{A function $g$ with a knot at $\tilde{x} \in (x_1, x_2)$.}
          \label{fig:g}
    \end{subfigure}%
    \begin{subfigure}{.5\textwidth}
          \centering
          \includegraphics[width=.9\linewidth]{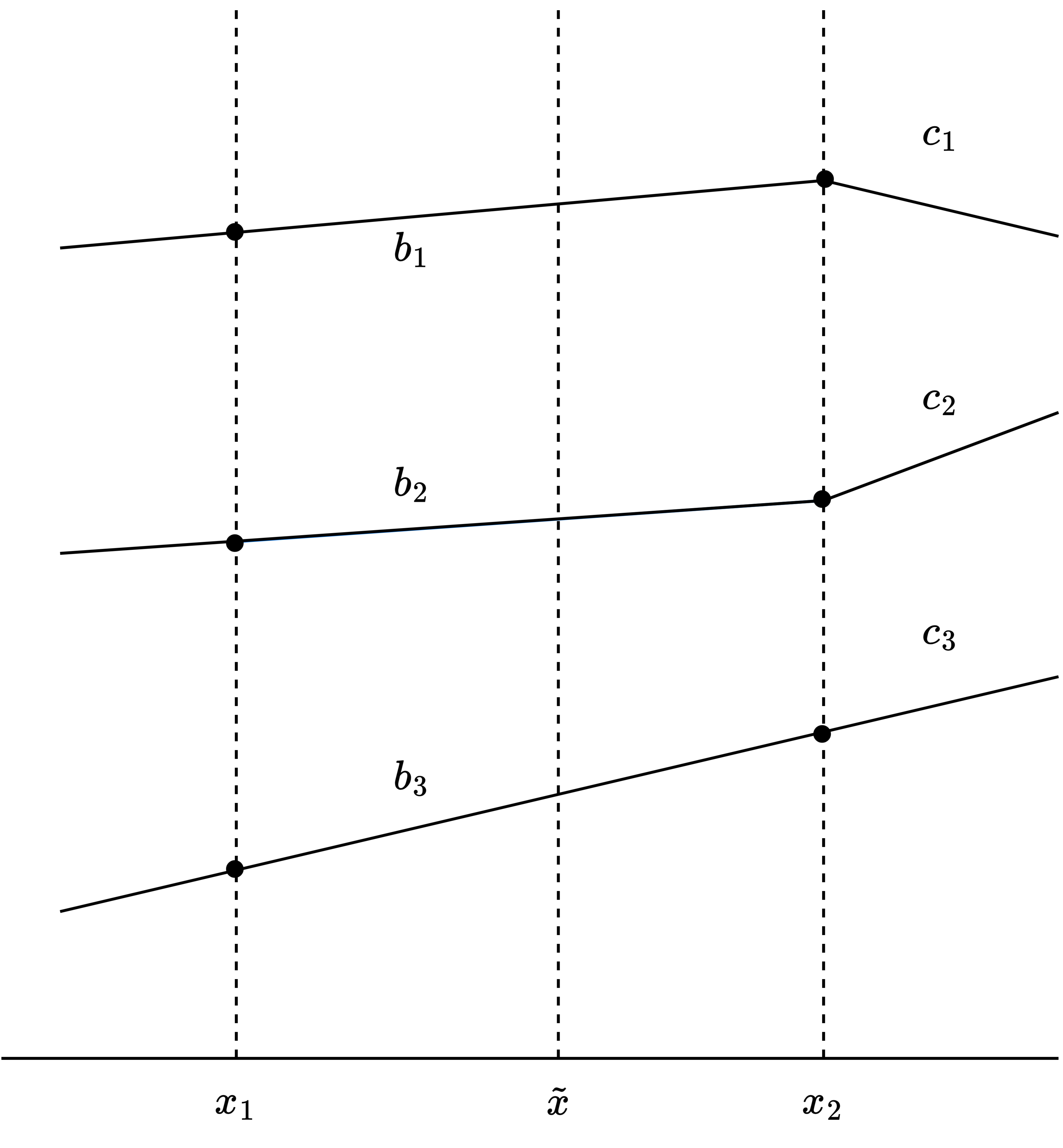}
          \caption{The function $f_{\cD}$.}
          \label{fig:f_D}
    \end{subfigure}
    \caption{Left: a function $g$ which has a knot in one or more of its outputs at a point $\tilde{x} \in (x_1, x_2)$. Right: the connect-the-dots interpolant $f_\cD$. The representational cost of $g$ is strictly greater than that of $f_{\cD}$.}
    \label{fig:univariate_proof_fig}
\end{figure}

Finally, consider the case where the extraneous knot is on the interval $[x_i, x_{i+1}]$ where $i = 1$ (the case $i = N-1$ follows by an analogous argument). Let $\va$ denote the vector of incoming slopes of $g$ at $x_1$. Define 
$$\vb = \frac{\vy_2 - \vy_1}{x_2 - x_1} \quad \vc = \frac{\vy_3 - \vy_2}{x_3 - x_2}.$$ Since $g$ has no extraneous knots except at $\tilde{x}$, the slopes of $g$ coming out of $x_2$ are $\vc$.  By optimality of $g$, it must be the case that $\va - \vb$ and $\vb - \vc$ are aligned (otherwise we could invoke \cref{lemma:keylemma} and strictly reduce the representational cost of $f$ by removing the knot at $\tilde{x}$, a contradiction), which implies that $\sgn(a_t - b_t) = \sgn(b_t - c_t)$ for each $t =1, \dots, T$. For any outputs $t$ which have a knot at $\tilde{x}$, this quantity is nonzero, in which case $|c_t - b_t| < |c_t - a_t|$ (see \cref{fig:g}). Let $ 1, \dots, t_0$ be the indices of the outputs which have a knot at $\tilde{x}$, and let $t_0 + 1, \dots, T $ be the indices of the outputs which do \emph{not} have a knot at $\tilde{x}$. We may again invoke \cref{lemma:keylemma} to remove the knots from $g$, resulting in a new function $\tilde{g}$ (satisfying $R(g) \geq R(\tilde{g})$) which has slopes $\va$ coming into $x_1$, $\vb$ between $x_1$ and $x_2$, and $\vc$ coming out of $x_2$. The contribution of these knots to $R(\tilde{g})$ is then given by:
\begin{align*}
    \left\| \begin{bmatrix}
        b_1 - a_1 \\ \vdots \\ b_{t_0} - a_{t_0}
    \end{bmatrix} \right\|_2 + \left\| \begin{bmatrix}
        c_1 - b_1 \\ \vdots \\ c_{t_0} - b_{t_0}
    \end{bmatrix} \right\|_2 &\geq \left\| \begin{bmatrix}
        b_1 - a_1 + c_1 - b_1 \\ \vdots \\ b_{t_0} - a_{t_0} + c_{t_0} - b_{t_0}
    \end{bmatrix} \right\|_2 \\
    &= \left\| \begin{bmatrix}
        c_1 - a_1 \\ \vdots \\ c_{t_0} - a_{t_0}
    \end{bmatrix} \right\|_2 \\
    &> \left\| \begin{bmatrix}
        c_1 - b_1 \\ \vdots \\ c_{t_0} - b_{t_0}
    \end{bmatrix} \right\|_2
\end{align*}
but the last quantity is exactly the contribution of these knots to $R(f_{\cD})$ (see \cref{fig:f_D}). This contradicts optimality of $\tilde{g}$, hence of $g$ and of $f$. The remainder of the proof is dedicating to showing that such datasets which admit non-unique solutions are rare when the data is randomly sampled from a continuous distribution.

\paragraph{Datasets which admit non-unique solutions have Lebesgue measure zero.} If $N = 2$ or $N = 3$, then $f_{\cD}$ is the only solution to \eqref{opt:pn}, so we focus on the case where $N \geq 4$. Suppose that for some $i = 2, \dots, N-2$, the data points $x_{i-1}, x_i, x_{i+1}, x_{i+2} \in \R$ and labels $\vy_{i-1}, \vy_{i}, \vy_{i+1}, \vy_{i+2} \in \R^{T}$ satisfy the requirement that
    \begin{align}
        \frac{\vy_{i+1}-\vy_i}{x_{i+1}-x_i}-\frac{\vy_i-\vy_{i-1}}{x_i-x_{i-1}} = w \left( \frac{\vy_{i+2}-\vy_{i+1}}{x_{i+2}-x_{i+1}} - \frac{\vy_{i+1} - \vy_i}{x_{i+1}-x_i}  \right)
    \end{align}
    for some $w > 0$, where both vectors are nonzero. After some computation, this is equivalent to the requirement that
    \begin{align}
        \underbrace{[\vy_{i-1}, \vy_i, \vy_{i+1}, \vy_{i+2}]}_{\mY_i \in \R^{T \times 4}} \Bigg( \underbrace{\begin{bmatrix}
            \frac{1}{x_{i}-x_{i-1}} \\ \frac{1}{x_{i+1}-x_i} - \frac{1}{x_i - x_{i-1}} \\ \frac{1}{x_{i+1}-x_i} \\ 0
        \end{bmatrix}}_{\va_1 \in \R^4} - w \underbrace{\begin{bmatrix}
            0 \\ \frac{1}{x_{i+1}-x_i} \\ -\frac{1}{x_{i+2} - x_{i+1}} - \frac{1}{x_{i+1} - x_i}  \\ \frac{1}{x_{i+2} - x_{i+1} }
        \end{bmatrix}}_{\va_2 \in \R^4} \Bigg) = \bm{0}
    \end{align}
    or equivalently, that $\mY \va_1 = w \mY \va_2$ for some $w > 0$. If this condition is satisfied, the matrix $\mU = \mY [\va_1, \va_2] \in \R^{T \times 2}$ has rank at most one, as does its Gram matrix $\mU \mU^\top$, which implies (because $T > 1$) that $\det(\mU \mU^\top) = 0$. Observe that $\det(\mU \mU^\top)$ is a real-valued non-trivial rational function of the variables $x_{i-1},  x_i, x_{i+1}, x_{i+2} \in \R$ and the entries of $\mY_i = [\vy_{i-1}, \vy_i, \vy_{i+1}, \vy_{i+2}] \in \R^{T \times 4}$. The zero set of this rational function is contained in the zero set of its polynomial numerator (an algebraic variety), and thus has Lebesgue measure zero in $\R^4 \times \R^{T \times 4}$, which implies that the set of possible $ x_{i-1}, x_i, x_{i+1}, x_{i+1}, \mY_{i}$ which satisfy the original condition for any $w > 0$ does as well. This holds for any $i = 2, \dots, N-2$, so the union of all such sets over $i = 2, \dots, N-2$ also has Lebesgue measure zero in $\R^{N} \times \R^{T \times N}$.
\end{proof}

\subsection{Proof of \cref{th:main_multivariate}} \label{appendix:proof_multivariate_main_theorem}

\begin{proof}
A cornerstone of our argument is the fact that the vectors $\vy_{\cdot,1}, \dots, \vy_{\cdot,T} \in \R^N$, generated according to the process described in \cref{sec:multivariate}, are exchangeable. Before proceeding, we note that there is an important nuance in characterizing the behavior of solutions to \eqref{opt:multivariate_problem_T_tasks} as random variables, which is that given a fixed (deterministic) dataset, the output weights which solve \eqref{opt:multivariate_problem_T_tasks} may not be unique. To account for this possibility, we assume that the optimization in \eqref{opt:multivariate_problem_T_tasks} is a (measurable) deterministic procedure $F$ which selects a single solution from the solution set, and this procedure is ``permutation invariant'' in the following sense: if $F$ selects $\vv_1^*, \dots, \vv_K^*$ as a solution for $\vy_1, \dots, \vy_N$, then $F$ also selects $\pi(\vv_1^*), \dots, \pi(\vv_K^*)$ as a solution for $\pi(\vy_1), \dots, \pi(\vy_N)$, where $\pi$ is any entry-wise permutation of the $T$ vector elements.
Indeed, if the solution to \eqref{opt:multivariate_problem_T_tasks} is unique, this permutation invariance property is fulfilled trivially since the regularization term in \eqref{opt:multivariate_problem_T_tasks} is itself permutation invariant. We then have the following:

\begin{proposition} \label{prop:exchangeable}
    Under the assumptions stated above: for each $k = 1, \dots, K$, the optimal output weights $v_{k1}^*, \dots, v_{kT}^*$ which solve \eqref{opt:multivariate_problem_T_tasks} are exchangeable random variables.
\end{proposition}
\begin{proof}
    For any $k = 1, \dots, K$ and any permutation $\pi$ of the indices $1, \dots, T$, we have:
    \begin{align}
        [v_{k \pi(1)}^*, \dots, v_{k \pi(T)}^*] &= F \left( \{ y_{i\pi(1)}, \dots, y_{i\pi(T)} \}_{i=1}^N  \right) \\
        &\overset{d}{=} F \left( \{ y_{i1}, \dots, y_{iT} \}_{i=1}^N  \right) \\
        &= [v_{k1}^*, \dots, v_{kT}^*]
    \end{align}
    The second line above follows from exchangeability of $y_{i1}, \dots, y_{iT}$, and $\overset{d}{=}$ denotes equivalence in distribution.
\end{proof}
Exchangeability of the $v_{k1}^*, \dots, v_{kT}^*$ (and hence of their squares $(v_{k1}^*)^2, \dots, (v_{kT}^*)^2$) provides us with the following fact\footnote{In the following discussion, the random variables $\frac{(v_{ks}^*)^2}{ \frac{1}{T} \sum_{t=1}^T (v_{kt}^*)^2}$ and $\frac{(v_{ks}^*)^2}{\sum_{t \neq s} (v_{kt}^*)^2}$ are understood to be $[0, \infty]$-valued, taking the value $\infty$ if and only if the denominator is equal to zero and the numerator is nonzero, and taking the value zero if both the numerator and denominator are equal to zero. \label{fn:inf_val_rv}}:
\begin{lemma} \label{lemma:exp_one}
    For any $s = 1, \dots, T$ and any $k = 1, \dots, K$:
    \begin{eqnarray}
        \bE \left[ \frac{(v_{ks}^*)^2}{ \frac{1}{T} \sum_{t=1}^T (v_{kt}^*)^2} \right] = 1
    \end{eqnarray}
\end{lemma}
\begin{proof}

    By exchangeability, the marginal distributions of $(v_{k1}^*)^2, \dots, (v_{kT}^*)^2$ are identical. Therefore we have that,
\begin{align}
    \bE \left[ (v_{ks}^*)^2 \Bigg\rvert \frac{1}{T} \sum_{t=1}^T (v_{kt}^*)^2 \right] = \bE \left[ (v_{kj}^*)^2 \Bigg\rvert \frac{1}{T} \sum_{t=1}^T (v_{kt}^*)^2 \right]
\end{align}
for any $s, j = 1, \dots, T $ and any $k = 1, \dots, K$. Therefore:
\begin{align}
    \frac{1}{T} \sum_{t=1}^T (v_{kt}^*)^2 &= \bE \left[ \frac{1}{T} \sum_{t=1}^T (v_{kt}^*)^2 \Bigg\rvert \frac{1}{T} \sum_{t=1}^T (v_{kt}^*)^2  \right] \\
    &= \frac{1}{T} \sum_{t=1}^T \bE \left[ (v_{kt}^*)^2 \Bigg\rvert \frac{1}{T} \sum_{t=1}^T (v_{kt}^*)^2  \right] \\
    &= \bE \left[ (v_{ks}^*)^2 \Bigg\rvert \frac{1}{T} \sum_{t=1}^T (v_{kt}^*)^2  \right]
\end{align}
for any $s = 1, \dots, T$, and thus
\begin{align}
    \bE \left[ \frac{(v_{ks}^*)^2}{\frac{1}{T} \sum_{t=1}^T (v_{kt}^*)^2} \Bigg\rvert \frac{1}{T} \sum_{t=1}^T (v_{kt}^*)^2  \right] = \frac{1}{\frac{1}{T} \sum_{t=1}^T (v_{kt}^*)^2} \bE \left[ (v_{kj}^*)^2 \Bigg\rvert \frac{1}{T} \sum_{t=1}^T (v_{kt}^*)^2 \right] = 1
\end{align}
which implies that
\begin{align}
    \bE \left[ \frac{(v_{ks}^*)^2}{\frac{1}{T} \sum_{t=1}^T (v_{kt}^*)^2} \right] = \bE \left[  \bE \left[ \frac{(v_{ks}^*)^2}{\frac{1}{T} \sum_{t=1}^T (v_{kt}^*)^2} \Bigg\rvert \frac{1}{T} \sum_{t=1}^T (v_{kt}^*)^2  \right]  \right] = \bE[1] = 1
\end{align}
\end{proof}

Applying Markov's inequality to \cref{lemma:exp_one} yields the following useful result:

\begin{lemma} \label{lemma:markov_bound}
    For any $s = 1, \dots, T$, any $k = 1, \dots, K$, and any $0 < \beta < 1$:
    \begin{align} \label{eq:Markov_bound}
    \bP \left( \frac{(v_{ks}^*)^2}{\sum_{t \neq s} (v_{kt}^*)^2} \geq \frac{1}{T^\beta} \right) \leq \frac{T^\beta + 1}{T}
\end{align}
\end{lemma}
\begin{proof}
    By Markov's inequality and \cref{lemma:exp_one}, we have
\begin{align}
    \bP \left( \frac{(v_{ks}^*)^2}{\frac{1}{T} \sum_{t=1}^T (v_{kt}^*)^2} \geq \epsilon \right) \leq \frac{\bE \left[ \frac{(v_{ks}^*)^2}{\frac{1}{T} \sum_{t=1}^T (v_{kt}^*)^2} \right]}{\epsilon} = \frac{1}{\epsilon}
\end{align}
for any $s$, $k$ and any $\epsilon > 0$. Equivalently:
\begin{align}
    \bP \left( \frac{(v_{ks}^*)^2}{\sum_{t \neq s} (v_{kt}^*)^2} \geq \frac{1}{T/\epsilon-1} \right) \leq \frac{1}{\epsilon}
\end{align}
The lemma follows by taking $\epsilon = T/(T^\beta + 1)$.
\end{proof}
\begin{lemma} \label{lemma:subvectors}
For any $s = 1, \dots, T$, any $k = 1, \dots, K$, and any $0 < \beta < 1$:
\begin{align}
    \left(1-T^{\beta-1} \right) \left( \sum_{t=1}^T (v_{kt}^*)^2 \right) \leq \sum_{t \neq s} (v_{kt}^*)^2 \leq \sum_{t=1}^T (v_{kt}^*)^2
\end{align}
with probability at least $1-T^{-\beta}$.
\end{lemma}
\begin{proof}
    The upper bound holds with probability one simply because all terms in the sum are nonnnegative. For the lower bound, note that again by Markov's inequality and \cref{lemma:exp_one}, we have
\begin{align}
    \bP \left( (v_{ks}^*)^2 \geq \frac{\epsilon}{T} \sum_{t=1}^T (v_{kt}^*)^2 \right) \leq \frac{1}{\epsilon}
\end{align}
for any $s$, $k$ and any $\epsilon > 0$. Equivalently:
\begin{align}
    \bP \left( \sum_{t=1}^T (v_{kt}^*)^2 - \sum_{t \neq s} (v_{kt}^*)^2  \geq \frac{\epsilon}{T} \sum_{t=1}^T (v_{kt}^*)^2 \right) = \bP \left( \sum_{t \neq s} (v_{kt}^*)^2 \leq \left( 1- \frac{\epsilon}{T} \right) \left( \sum_{t=1}^T (v_{kt}^*)^2 \right) \right) \leq \frac{1}{\epsilon}
\end{align}
Therefore:
\begin{align}
    \bP \left( \sum_{t \neq s} (v_{kt}^*)^2 \geq \left( 1- \frac{\epsilon}{T} \right) \left( \sum_{t=1}^T (v_{kt}^*)^2 \right) \right) \geq 1- \frac{1}{\epsilon}
\end{align}
The result follows by taking $\epsilon = T^\beta$.
\end{proof}

Additionally, we will use the following fact:
\begin{lemma} \label{lemma:v_bounded}
    Suppose that each label $y_{it}$ satisfies $|y_{it}| \leq B$ almost surely for some constant $B$ independent of $T$. Then each $v_{kt}^*$ solving \eqref{opt:multivariate_problem_T_tasks} satisfies $|v_{kt}^*| \leq \sqrt{CT}$, where $C$ is a constant independent of $T$. 
\end{lemma}
\begin{proof}
    Consider a set of single-task labels $y_1, \dots, y_N \in [-B,B]$ and let $\vv^* = [v_1^*, \dots, v_K^*]$ be an optimal set of output weights for \eqref{opt:multivariate_problem_T_tasks} with these labels (taking $T = 1$). Let $\vu^*$ be a solution to
    \begin{align} \label{opt:l1_fixed_dict}
        \min_{\vu \in \R^K} \| \vu \|_1 \qquad \mathrm{s.t.} \qquad \mD \vu = \vy
    \end{align}
    where $\vy = [y_1, \dots, y_N]^\top$ and $\mD \in \R^{N \times K}$ is a matrix whose $i$, $k^\textrm{th}$ entry is $\mD_{ik} = \left( \vw_k^\top \vx_i + b_k \right)_+$, where $\vw_k$, $b_k$ are chosen such that $\vy \in \text{range}(\mD)$, which guarantees that the problem above is well-posed. 
    Then $\| \vv^* \|_1 \leq \| \vu^* \|_1$. To see this, note that $\| \vu^* \|_1$ is an upper bound on the optimal objective value $S^*$ of the neural network interpolation problem \eqref{opt:pn} for single-task labels $y_1, \dots, y_N$, since \eqref{opt:pn} allows for optimizing over the $\vw_k, b_k$ in addition to the $v_k$. Moreover, $S^* \geq \| \vv^* \|_1$, since otherwise the optimal parameters for \eqref{opt:pn} would provide a smaller-than-optimal objective value for \eqref{opt:multivariate_problem_T_tasks}. Additionally, $\| \vu^* \|_1 \leq \| \mD^+ \vy \|_1$, where $\mD^+ = \mD^\top (\mD \mD^\top)^{-1}$ is the pseudo-inverse of $\mD$, this is because the vector $\tilde{\vu} = \mD^+ \vy$ necessarily satisfies the interpolation constraint in \eqref{opt:l1_fixed_dict}. Furthermore, 
    $$\| \mD^+ \vy \|_1 \leq \sqrt{K} \| \mD^+ \vy \|_2 \leq \sqrt{K} \| \mD^+ \|_2 \| \vy \|_2 = \sqrt{K} \| \vy \|_2 / \sigma_{\min}(\mD), $$
    where $\sigma_{\min}(\mD)$ denotes the smallest non-zero singular value of $\mD$. Combining these inequalities with boundedness of the $y_i$, we have $|v_k^*| \leq C$, where $C:= B \sqrt{KN}/\sigma_{\min}(\mD)$.

Now consider the full $T$-task problem with labels $\vy_1, \dots, \vy_N \in \R^T$. Note that any network $f_{\vtheta}$ with width $K \geq NT$ can represent the function $(f_{\vtheta_1^*}, \dots, f_{\vtheta_T^*})$, where each $f_{\vtheta_t^*}$ is an optimal solution with no more than $N$ neurons to the single-task problem \eqref{opt:multivariate_problem_T_tasks} for task-$t$ labels $y_{1t}, \dots, y_{Nt}$. Therefore, the optimal output weights $\vv_1^*, \dots, \vv_N^*$ solving \eqref{opt:multivariate_problem_T_tasks} for $\vy_1, \dots, \vy_N$ must satisfy $\sum_{k=1}^K \| \vv_k^* \|_2 \leq  \sum_{k=1}^K \| \vv_k^* \|_1  \leq CT$.  This upper bound also holds for solutions to \eqref{opt:multivariate_problem_T_tasks} when $K \geq N^2$, for which the optimal objective value is the same as when $K \geq NT$ (see footnote \footref{fn:K_N_squared} and Theorem 5 in \cite{shenouda2024variation}). Furthermore, by the neural balance theorem \cite{shenouda2024variation, yang2022better, parhi2023deep} the optimal input and output weights for \eqref{opt:multivariate_problem_T_tasks} must satisfy
    \begin{align*}
        \frac{1}{2} \sum_{k=1}^K \| \vw_k^* \|_2^2 + \| \vv_k^* \|_2^2 = \sum_{k=1}^K \| \vv_k^*\|_2 \|\vw^{*}_k\|_2 \leq CT
    \end{align*}
    and since each $\vw_k^* \in \bS^{d-1}$, the above implies that
    \begin{align*}
        |v_{kt}^*|^2 \leq \sum_{k=1}^K \| \bm{v}_k^* \|_2^2 \leq 2CT - K \leq 2CT \implies |v_{kt}^*| \leq \sqrt{2CT}
    \end{align*}
    for each $k$, $t$.
\end{proof}

\paragraph{Claim (i): $J$ is closely approximated by $H$ in a neighborhood of $J$'s minimizer.} With these results in hand, take $\vtheta^*$ to be a set of optimal \textit{active} neuron parameters as in the statement of \cref{th:main_multivariate}. For an individual $k$, define $g(v_{ks}):= \left\| \begin{bmatrix}
    v_{ks} \\ \vv_{k \setminus s}^*
\end{bmatrix} \right\|_2 = \sqrt{v_{ks}^2 + \| \vv_{k \setminus s}^* \|_2^2}$. Note that $g$ is infinitely differentiable everywhere unless $\| \vv_{k \setminus s}^* \|_2 = 0$. Since $\vtheta^*$ is restricted to active neurons, this is only possible if $v_{ks}^* \neq 0$, in which case $ \frac{(v_{ks}^*)^2}{\| \vv_{k \setminus s}^* \|_2^2} = \infty$ (see footnote \footref{fn:inf_val_rv}). Therefore, in the event that $\frac{(v_{ks}^*)^2}{\| \vv_{k \setminus s}^* \|_2^2} \leq T^{-\beta}$ (for some fixed $0 < \beta < 1$)---which occurs with probability at least $1-(T^\beta+1)/T$ by \cref{lemma:markov_bound}---we can express $g$ (for this $k$) using Taylor's theorem as
\begin{align} \label{eq:g_Taylor}
    g(v_{ks}) &= g(0) + g'(0) v_{ks} + \frac{1}{2} g''(0) v_{ks}^2 + \frac{1}{6} g'''(c) v_{ks}^3 \\
    &= \| \vv_{k \setminus s}^* \|_2 + \frac{1}{2} \frac{v_{ks}^2}{ \| \vv_{k \setminus s}^* \|_2} -\underbrace{ \frac{1}{2} \frac{c \| \vv_{k \setminus s}^* \|_2^2 \, v_{ks}^3}{ ( \| \vv_{k \setminus s}^*\|_2^2  + c^2 )^{5/2}}}_{R(v_{ks})}
\end{align}
for some $c \in (-|v_{ks}|, |v_{ks}|)$, where
\begin{align}
    |R(v_{ks})| \leq \frac{|c| |v_{ks}|^3}{2 \| \vv_{k \setminus s}^* \|_2^3} \leq \frac{|c|T^{3 \alpha}}{2} \left( \frac{(v_{ks}^*)^2}{\| \vv_{k \setminus s}^* \|_2^2} \right)^{3/2} \leq \frac{|c|T^{3 \alpha-3 \beta/2}}{2}
\end{align}
whenever $|v_{ks}| \leq T^{\alpha}|v_{ks}^*|$, for any $\alpha > 0$. By \cref{lemma:v_bounded}, $|c| \leq T^\alpha|v_{ks}^*| \leq T^\alpha\sqrt{CT} = T^{\alpha+1/2} \sqrt{C}$ for a constant $C$ independent of $T$. Therefore:
\begin{align}
    |R(v_{ks})| \leq \frac{ \sqrt{C}}{2} T^{4\alpha + 1/2 - 3 \beta/2}
\end{align}
with probability at least $1-(T^{\beta} + 1)/T $ whenever $|v_{ks}| \leq T^{\alpha} |v_{ks}^*|$. By the union bound:
\begin{align}
    \sum_{k=1}^K |R(v_{ks})| \leq \frac{K\sqrt{C}}{2} T^{4\alpha + 1/2 - 3 \beta/2} = O \left( T^{4\alpha + 1/2 - 3 \beta/2} \right)
\end{align}
with probability at least $1-K(T^\beta+1)/T = 1-O \left( T^{\beta-1} \right)$ whenever $|v_{ks}| \leq T^{\alpha} |v_{ks}^*|$ for all $k = 1, \dots, K$. Since $J$ and $H$ in the theorem statement differ only by $\sum_{k=1}^K |R(v_{ks})|$, this concludes the proof of the first claim, taking $\beta = 2/3$ and $\alpha = 1/16$.

\paragraph{Claim (ii): the exact minimizer of $H$ is a near-minimizer of $J$.} Note that
\begin{align}
    \sqrt{v_{ks}^2 + \| \vv_{k \setminus s}^* \|_2^2} &\leq \| \vv_{k \setminus s}^* \|_2 + \frac{v_{ks}^2}{2 \| \vv_{k \setminus s}^* \|_2} \\
    \iff v_{ks}^2 + \| \vv_{k \setminus s}^* \|_2^2 &\leq \left( \| \vv_{k \setminus s}^* \|_2 + \frac{v_{ks}^2}{2 \| \vv_{k \setminus s}^* \|_2} \right)^2 \\
    &= \| \vv_{k \setminus s}^* \|_2^2 + v_{ks}^2 + \frac{v_{ks}^4}{4 \| \vv_{k \setminus s}^* \|_2^2} \\
    \iff 0 &\leq \frac{v_{ks}^4}{4 \| \vv_{k \setminus s}^* \|_2^2} 
\end{align}
which is always true for any $v_{ks}$. Therefore, $J(v_{1s}, \dots, v_{Ks}) \leq H(v_{1s}, \dots, v_{Ks})$ for any $v_{1s}, \dots, v_{Ks}$. For brevity, denote $\vv_s := (v_{1s}, \dots, v_{Ks})$, $\vv_s^*:= (v_{1s}^*, \dots, v_{Ks}^*)$, and $\vv_s':= (v_{1s}', \dots, v_{Ks}')$. Write
\begin{align}
    J(\vv') - J(\vv^*) = J(\vv') - H(\vv') + H(\vv') - H(\vv^*) + H(\vv^*) - J(\vv^*)
\end{align}
Because $J(\vv) \leq H(\vv)$ for any $\vv$, $J(\vv') - H(\vv') \leq 0$, and because $\vv'$ minimizes $H$, $H(\vv') - H(\vv^*) \leq 0$. By claim (i) in the proof, $H(\vv^*) - J(\vv^*) \leq \frac{K\sqrt{C}}{2} T^{4\alpha + 1/2 - 3 \beta/2}$. Therefore: 
\begin{align}
    J(\vv') - J(\vv^*) \leq \frac{K\sqrt{C}}{2} T^{4\alpha + 1/2 - 3 \beta/2}  = O \left( T^{4\alpha + 1/2 - 3 \beta/2} \right)
\end{align}
with probability at least $1-K(T^\beta+1)/T = 1-O \left( T^{\beta-1} \right)$.
\end{proof}
\section{Additional Experimental Results and Details}\label{appendix:additional_experiments}
All of our experiments were carried out in PyTorch and used the Adam optimizer. We trained the models with mean squared error loss and weight decay with $\lambda=1e-5$ for the univariate experiments and $\lambda=1e-3$ for the multi-variate experiments. All models were trained to convergence. The networks were initialized with $20$ neurons for the univariate experiments and $800$ neurons for the multi-variate experiments. For solving \eqref{eq:H} over a fixed set of optimal neurons we utilized CVXPy \cite{diamond2016cvxpy}.
\subsection{Additional Experiments from \Cref{sec:multivariate}}

\cref{fig:extra-multivariate-exps} below provides additional experimental results to accompany the discussion in \cref{sec:multivariate}. The results demonstrate that our observations from \cref{sec:multivariate} are consistent across multiple random initializations of the network.
\begin{figure}[hb]
    \centering
    \begin{subfigure}{0.3\textwidth}
        \includegraphics[width=\textwidth]{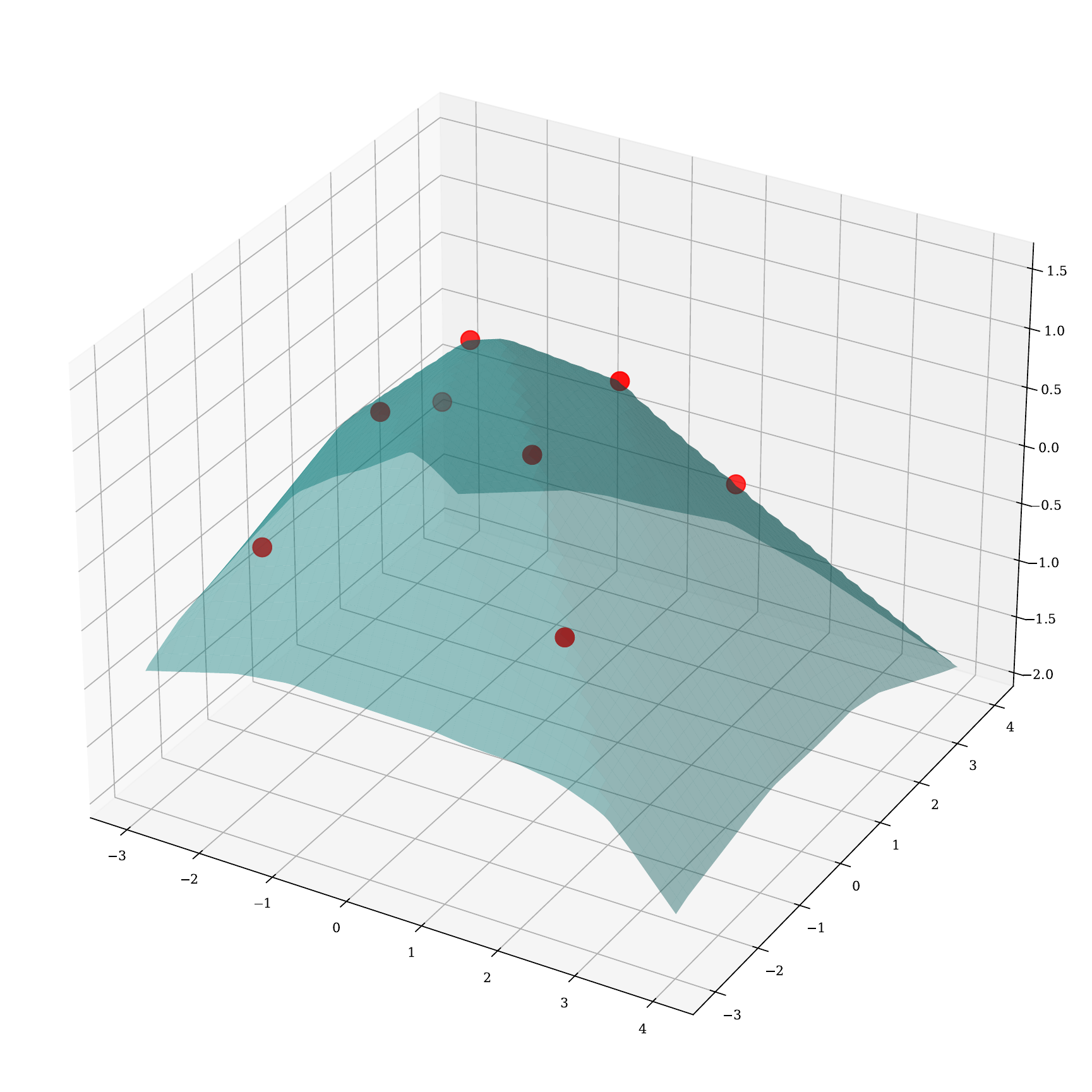}
        \caption{}
    \end{subfigure}
    \hfill
    \begin{subfigure}{0.3\textwidth}
        \includegraphics[width=\textwidth]{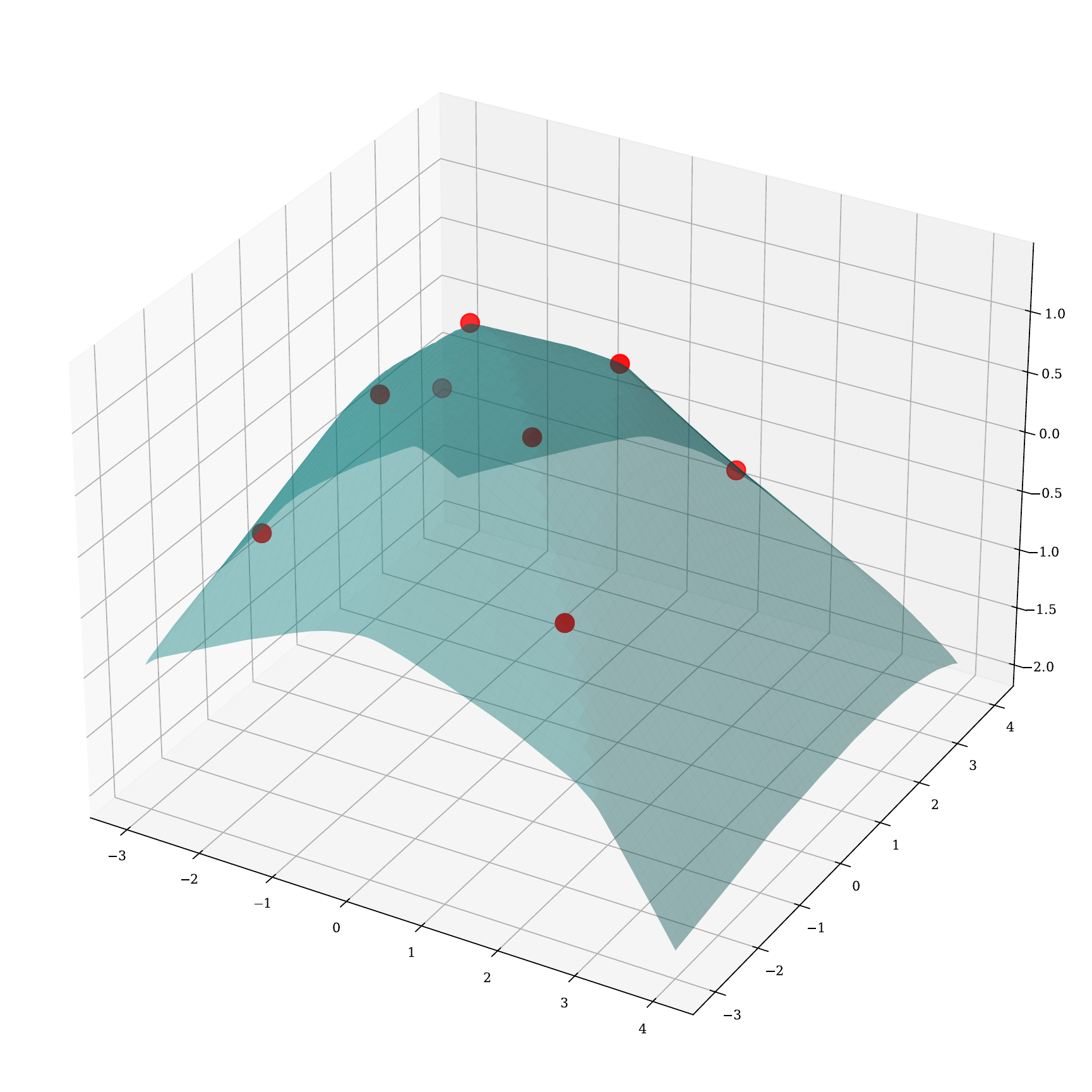}
        \caption{}
    \end{subfigure}
    \hfill
    \begin{subfigure}{0.3\textwidth}
        \includegraphics[width=\textwidth]{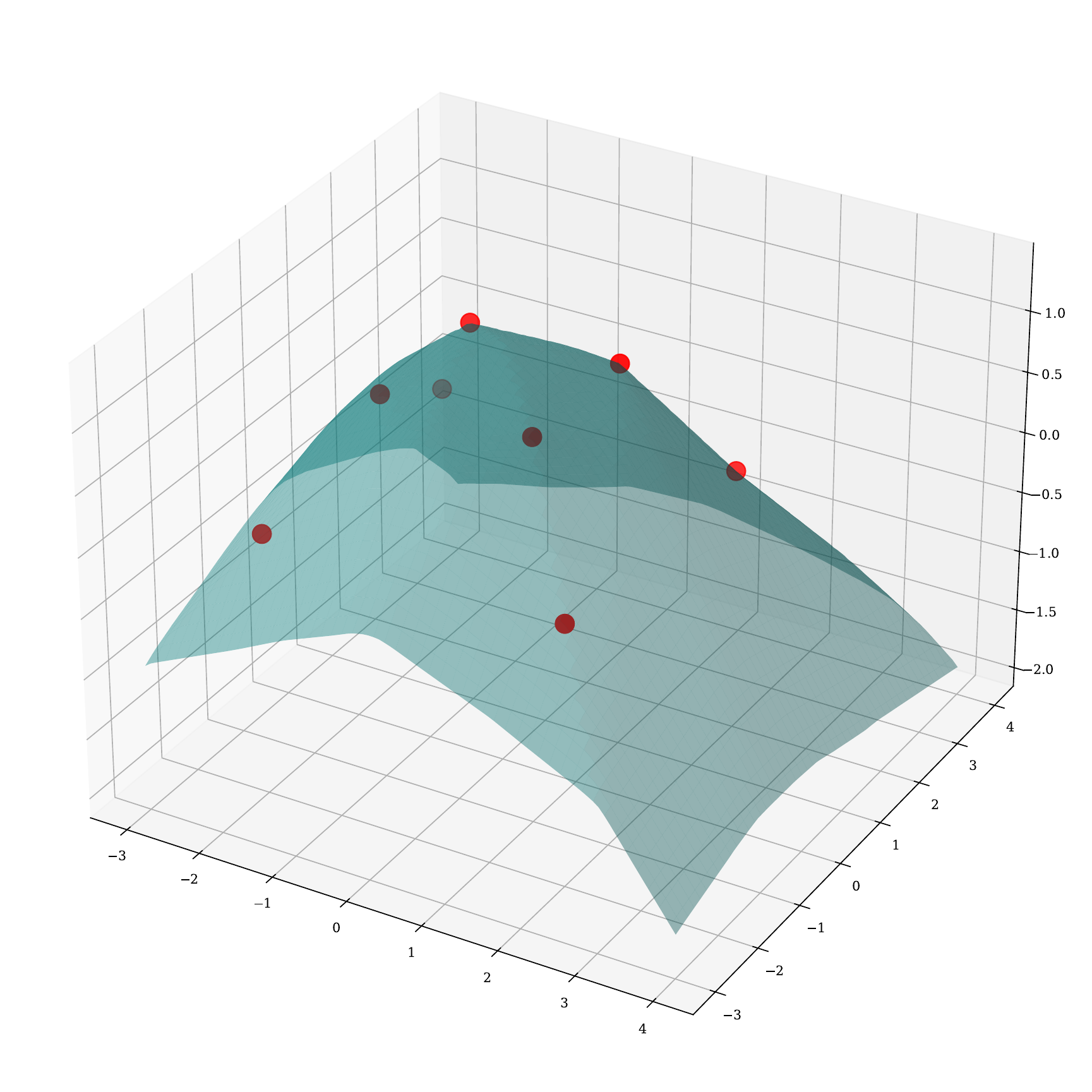}
        \caption{}
    \end{subfigure}
    \bigskip
    \begin{subfigure}{0.3\textwidth}
        \includegraphics[width=\textwidth]{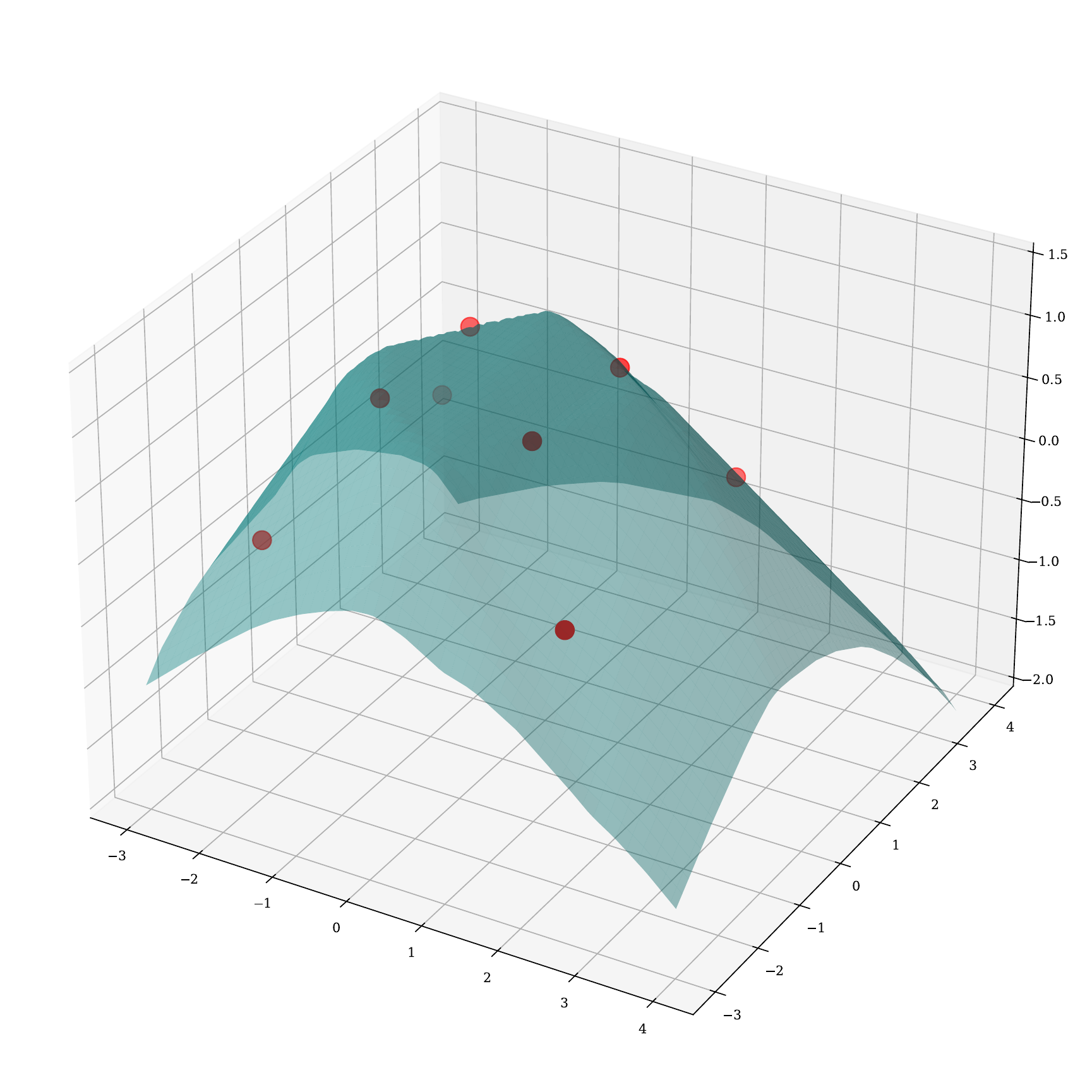}
        \caption{}
    \end{subfigure}
    \hfill
    \begin{subfigure}{0.3\textwidth}
        \includegraphics[width=\textwidth]{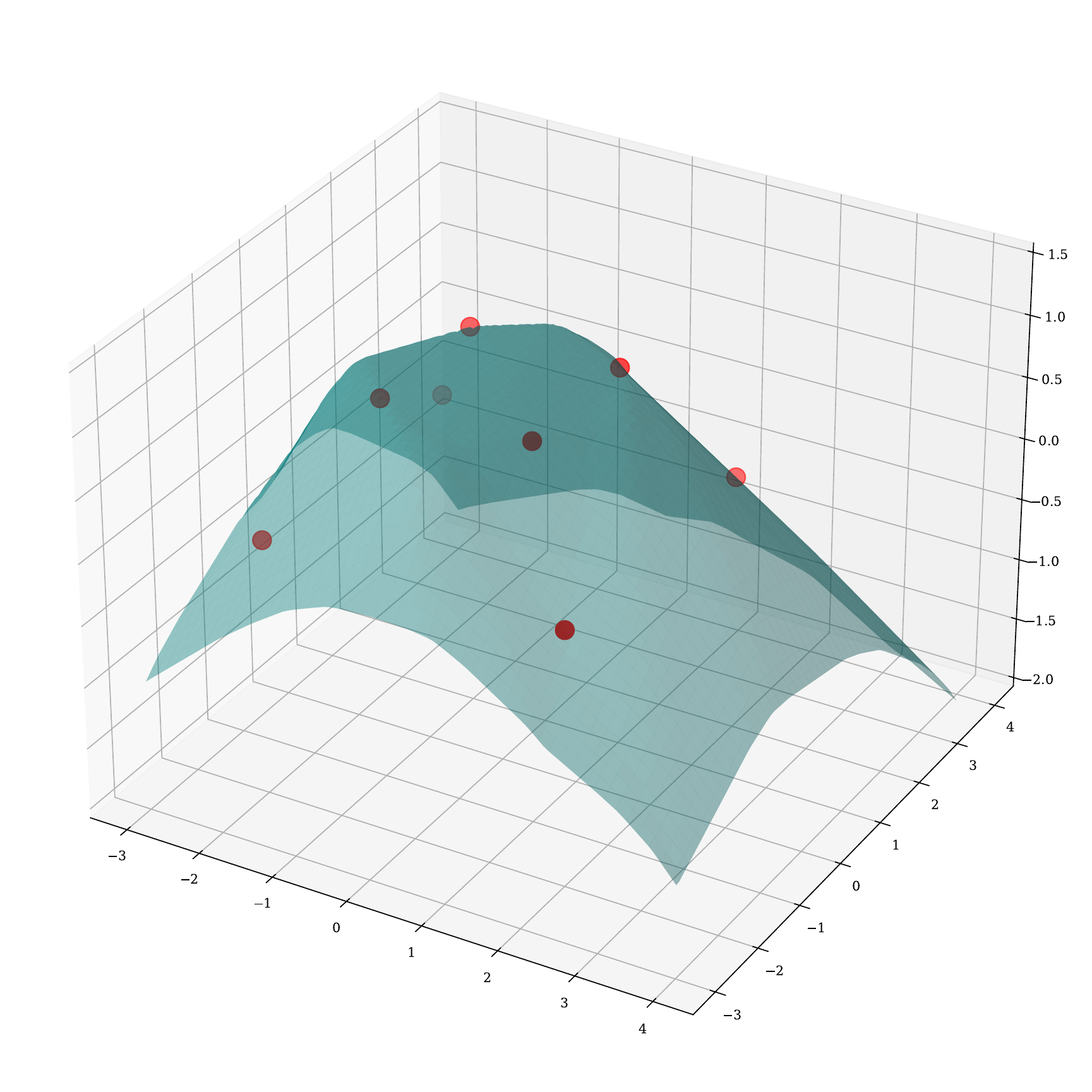}
        \caption{}
    \end{subfigure}
    \hfill
        \begin{subfigure}{0.3\textwidth}
        \includegraphics[width=\textwidth]{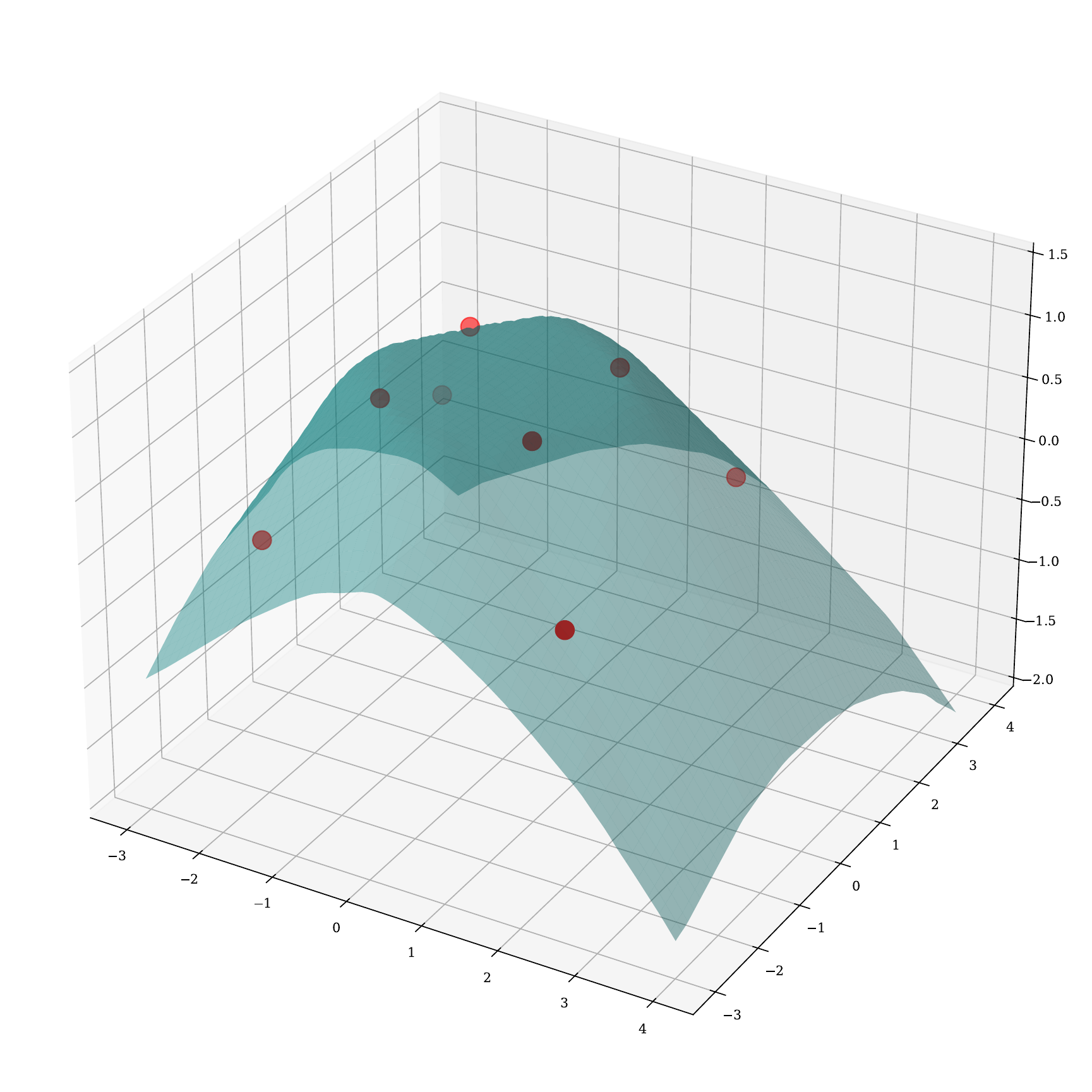}
        \caption{}
    \end{subfigure}
    \caption{We present three more trials of the same experiment from \cref{sec:multivariate}. The top row corresponds to the solution of the fist output of a multi-task neural network with $T=101$ tasks. The first task is the original (i.e. interpolating the red points), the other $100$ are randomly sampled i.i.d from a Bernoulli distribution with equal probability for one and zero. The second row corresponds to the solutions obtained by solving \eqref{opt:RKHS_problem}. We see again that for the $T=101$ multi-task neural network the learned function is consistent across multiple random initializations. Moreover, those solutions are also similar to the ones obtained by solving \eqref{opt:RKHS_problem}. These results suggest that with many diverse tasks, the contributions of any one task on the optimal neurons are not significant.}
    \label{fig:extra-multivariate-exps}
\end{figure}
\subsection{High Dimensional Multivariate Experiments}
In this section we provide additional experiments in a higher dimensional setting to demonstrate how multi-task solutions can differ drastically from single-task. For these experiments we consider a student-teacher model. In particular, we generated $25$ random ReLU neurons with unit norm input weights $\vw_t \in \R^{5}$ for $t = 1,\dots, 25$. These served as ``teacher" neurons. We then generated a multi-task dataset $\{\vx_i, \vy_i\}^{20}_{i=1}$ with $\vx_i \in \R^{5}$ and sampled i.i.d from a standard multi-variate Gaussian distribution. The labels $\vy_i \in \R^{25}$ were then generated according to the teacher ReLU neurons, that is,
\begin{align*}
    y_{it} = (\vw^{T}_t \vx_i)_{+}.
\end{align*}
We then trained 25 student single-output ReLU neural networks on each tasks as well as a 25-output multi-task ReLU neural network on all the tasks. Both were trained to minimize MSE loss and were regularized a weight decay parameter of $\lambda = 1e-4$. All networks nearly interpolated the data with MSE value less than $1e-4$. \Cref{fig:single_task_nets} shows the learned single task networks evaluated along a a unit norm vector $\vw \in \R^{5}$. From the plots it is clear that the single task networks recover the ground truth function (i.e. a single ReLU neuron) as it looks like a ReLU ridge function in every direction. Moreover, we observed an average sparsity of five active neurons across all the trained single-output networks.  

In the \Cref{fig:mtl_high_dim_nets} we also plot the  output of the $t^{\text{th}}$ function from the learned multi-task network evaluated at the same $\vw$. In this case, the functions look very different from a ReLU ridge function and do not recover the ground truth data-generating function for the respective task. \Cref{fig:mtl_sparsity} shows the sparsity pattern of the weights for each neuron with roughly $150$ neurons contributing to all the outputs.
\begin{figure}
    \centering
    \includegraphics[width=\linewidth]{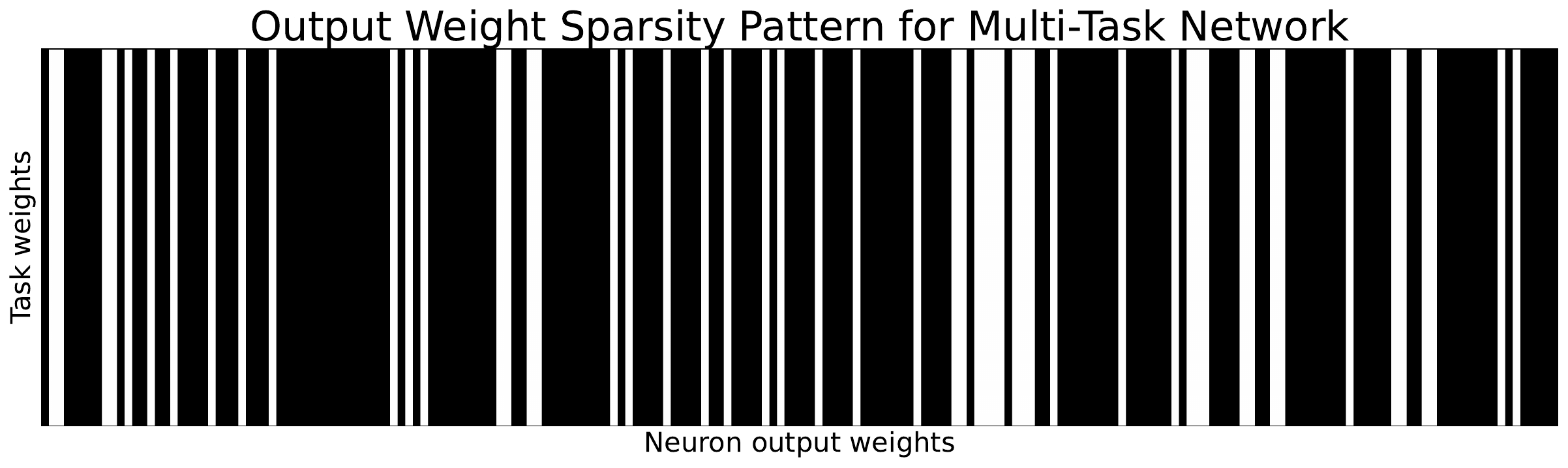}
    \caption{Sparsity pattern for output weight matrix of the multi-task student network. The $k^{\text{th}}$ column in the matrix corresponds to the output weight of the $k^{\text{th}}$ neuron. We observe that each neuron either contributes to all the tasks or none.}
    \label{fig:mtl_sparsity}
\end{figure}
\begin{figure}
    \centering
    \includegraphics[width=\linewidth]{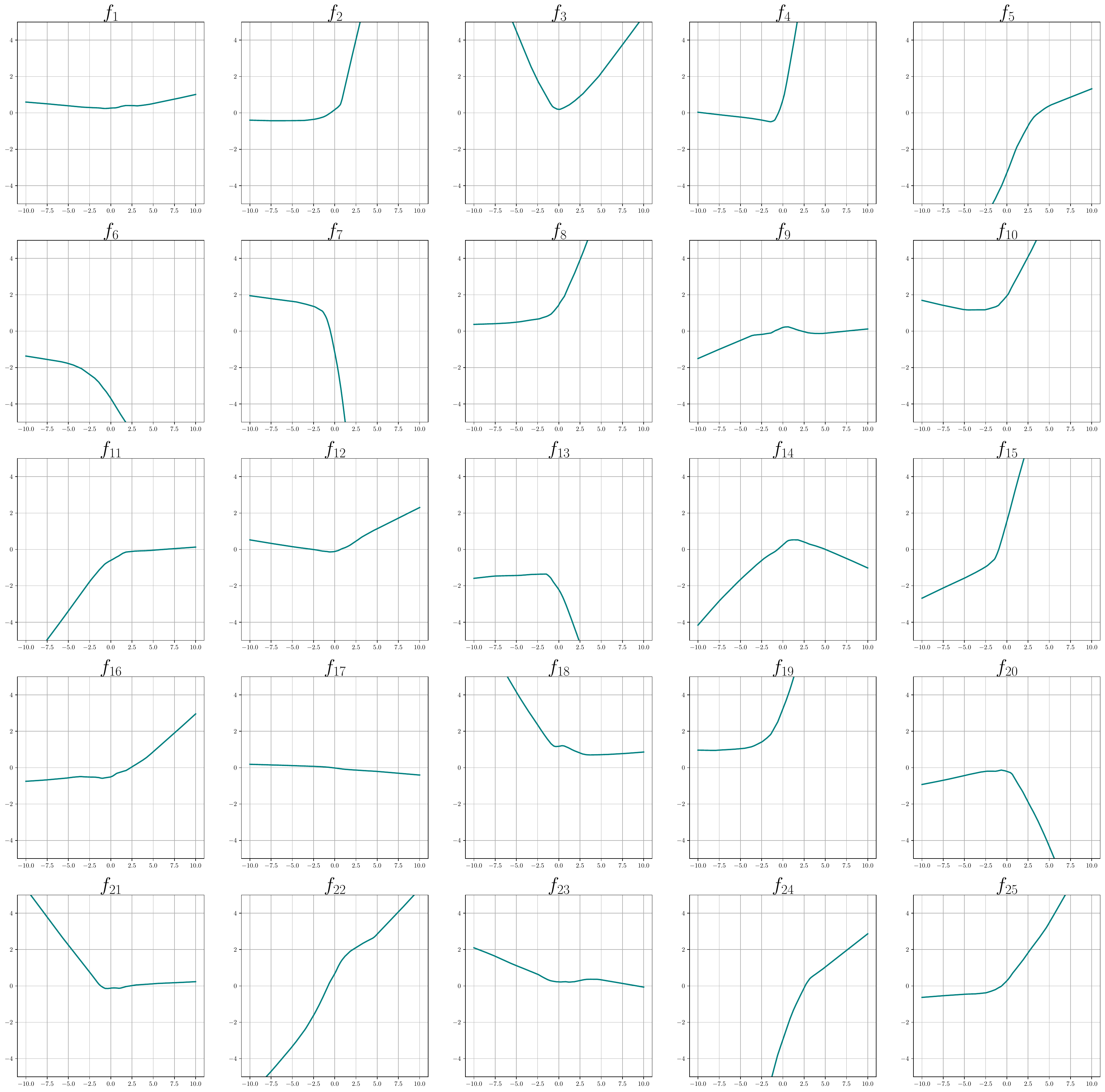}
    \caption{Multi-task solutions along the same direction $\vw$. Here $f_{t}$ denotes the $t^{\text{th}}$ output of the multi-task network. We observe that unlike \Cref{fig:single_task_nets} the functions do not look like ReLU ridge functions and have variation in multiple directions.}
    \label{fig:mtl_high_dim_nets}
\end{figure}
\begin{figure}
    \centering
    \includegraphics[width=\linewidth]{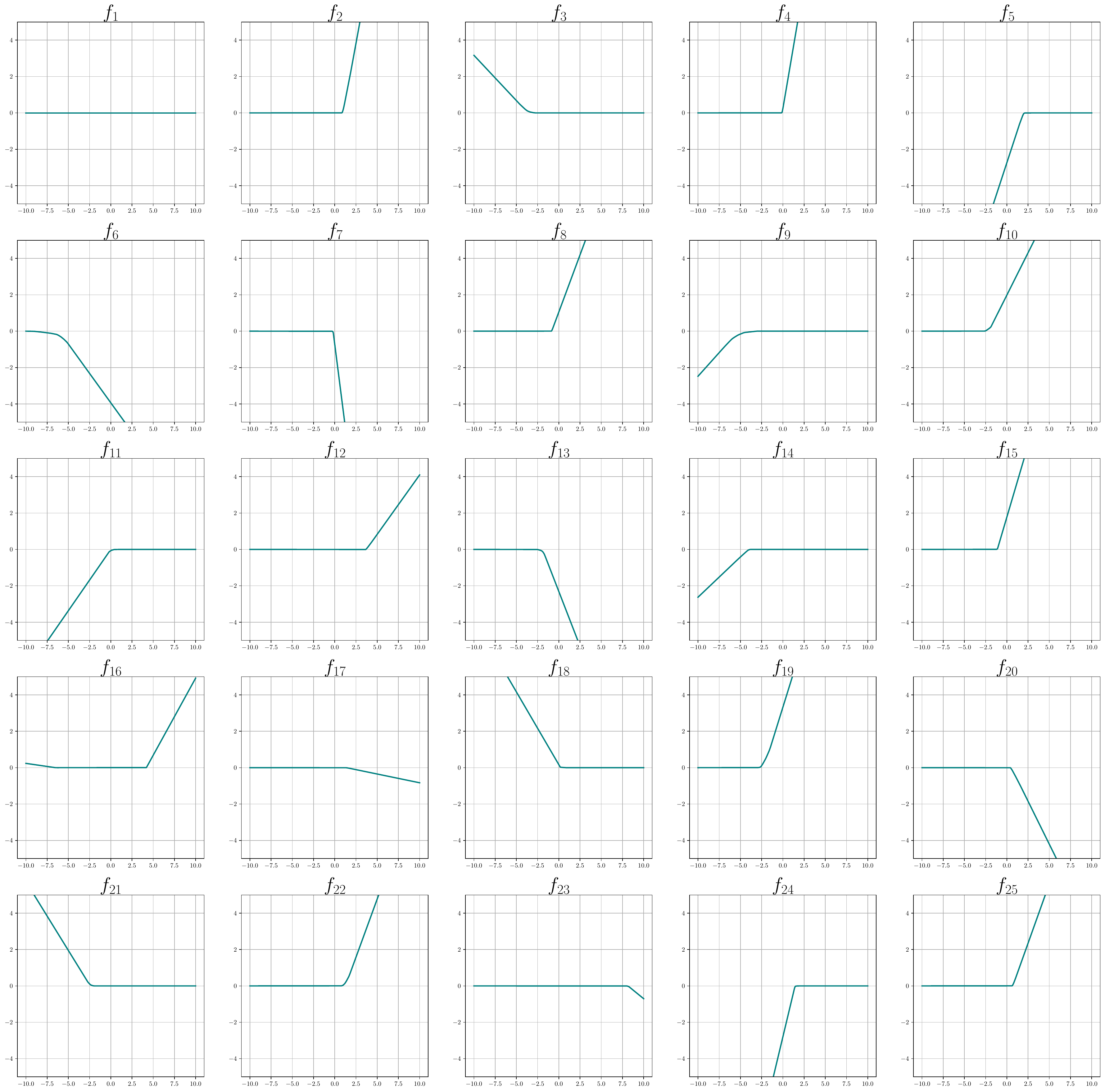}
    \caption{The 25 single-task networks evaluated along the same direction $\vw$ as in \Cref{fig:mtl_high_dim_nets}. Here $f_{t}$ denotes the $t^{\text{th}}$ single-task network trained on task $t$ according to the data generating function above.  Here as we expect the single-task nets are ReLU ridge functions. We note that these observations hold across different choices of the one-dimensional subspace $\vw$.}
    \label{fig:single_task_nets}
\end{figure}
\end{appendices}
\end{document}

%% file: math_commands.tex
\usepackage{amsmath,amsfonts,bm,mathtools}
\usepackage{bbm}
\usepackage[mathscr]{eucal}

\def\1{\bm{1}}

\def\<{\langle}
\def\>{\rangle}

\def\vtheta{{\bm{\theta}}}

\def\va{{\bm{a}}}
\def\vb{{\bm{b}}}
\def\vc{{\bm{c}}}

\def\vs{{\bm{s}}}

\def\vu{{\bm{u}}}
\def\vv{{\bm{v}}}
\def\vw{{\bm{w}}}
\def\vx{{\bm{x}}}
\def\vy{{\bm{y}}}

\def\mA{{\bm{A}}}

\def\mD{{\bm{D}}}

\def\mQ{{\bm{Q}}}

\def\mU{{\bm{U}}}

\def\mY{{\bm{Y}}}

\def\mPhi{{\bm{\Phi}}}

\DeclareMathAlphabet{\mathsfit}{\encodingdefault}{\sfdefault}{m}{sl}
\SetMathAlphabet{\mathsfit}{bold}{\encodingdefault}{\sfdefault}{bx}{n}

\def\cD{{\mathcal{D}}}

\def\cH{{\mathcal{H}}}

\def\cL{{\mathcal{L}}}
\def\cM{{\mathcal{M}}}

\def\bE{{\mathbb{E}}}

\def\bP{{\mathbb{P}}}

\def\bR{{\mathbb{R}}}
\def\bS{{\mathbb{S}}}

\newcommand{\R}{\mathbb{R}}

\DeclareMathOperator*{\argmin}{arg\,min}